\documentclass{jmlr} 

\title[ER-SPUD]{Exact Recovery of Sparsely-Used Dictionaries}

 \author{\Name{Daniel A. Spielman} \Email{spielman@cs.yale.edu}
\and
 \Name{Huan Wang} \Email{huan.wang@yale.edu}\\
 \addr Department of Computer Science,
Yale University
\AND
 \Name{John Wright} \Email{johnwright@ee.columbia.edu}\\
 \addr Department of Electrical Engineering,
Columbia University
 }

\usepackage{algorithm, algorithmic,verbatim,graphicx, fancybox}

\setcitestyle{numbers,square}

\renewcommand{\Re}{\mathbb{R}}
\newcommand{\E}{\mathbb{E}}
\renewcommand{\P}{\mathbb{P}}

\newcommand{\mc}{\mathcal}
\newcommand{\eps}{\varepsilon}

\newcommand{\event}{\mc E}

\newcommand{\mb}{\boldsymbol}

\newcommand{\prob}[1]{\P\left[ \, #1 \, \right] }
\newcommand{\condprob}[2]{\P\left[ \, #1 \mid #2 \, \right] }

\newcommand{\magnitude}[1]{\left|#1\right|}

\def\abs#1{\left|#1  \right|}
\def\norm#1{\left\| #1 \right\|}

\newdimen\pIR
\newcommand\StevesR{{\rm I\kern\pIR R}}
\def\Reals#1{\StevesR^{#1}}

\def\sizeof#1{\left|#1  \right|}
\def\setof#1{\left\{#1  \right\}}

\def\mand{\mbox{ and }}

\def\mnot{\mbox{ not}}

\def\floor#1{\left\lfloor #1 \right\rfloor}
\def\ceil#1{\left\lceil #1 \right\rceil}
\newif\ifincludefigures
\includefiguresfalse

\begin{document}

\maketitle

\begin{abstract}
We consider the problem of learning sparsely used dictionaries with an arbitrary square
dictionary and a random, sparse coefficient matrix. We prove that $O
(n \log n)$ samples are sufficient to uniquely determine the
coefficient matrix. Based on this proof, we design a polynomial-time
algorithm, called Exact Recovery of Sparsely-Used Dictionaries
(ER-SpUD), and prove that it probably recovers the dictionary and
coefficient matrix when the coefficient matrix is sufficiently
sparse. Simulation results show that ER-SpUD reveals the true
dictionary as well as the coefficients with probability higher than
many state-of-the-art algorithms.
\end{abstract}


\section{Introduction}

In the Sparsely-Used Dictionary Learning Problem, one is given a matrix $\mb Y \in \Re^{n \times p}$
  and asked to find a pair of 
  matrices $\mb A \in \Re^{n \times m}$ and $\mb X \in \Re^{m \times p}$ 
  so that $\norm{\mb Y - \mb A \mb X}$ is small and so that $\mb X$ is {\em sparse}  --
  $\mb X$ has only a few nonzero elements.
We examine solutions to this problem in which $\mb A$ is a basis, so $m = n$,
  and without the presence of noise, in which case we insist $\mb Y = \mb A \mb X$.
Variants of this problem arise in different contexts in
machine learning, signal processing, and even computational
neuroscience. We list two prominent examples:
\begin{itemize}
\item {\em Dictionary learning \cite{Olshausen1996-Nature,Kruetz-Delgado2003-NC}:} Here, the goal is to find a basis $\mb A$ that most compactly represents a given set of sample data. Techniques based on learned dictionaries have performed quite well in a number of applications in signal and image processing \cite{Bruckstein2009-SIAM,Rubenstein2010-IEEE,Yang2010-TIP}.

\item {\em Blind source separation \cite{Zibulevsky2001-NC}:} Here, the rows of $\mb X$ are considered the emissions of various sources over time. The sources are linearly mixed by $\mb A$ (instantaneous mixing). Sparse component analysis \cite{Zibulevsky2001-NC,Georgiev2005-TNN} is the problem of using the prior information that the sources are sparse in some domain to unmix $\mb Y$ and obtain $(\mb A,\mb X)$.
\end{itemize}
These applications raise several basic questions. First, when is the problem well-posed? More precisely, suppose that $\mb Y$ is indeed the product of some unknown dictionary $\mb A$ and sparse coefficient matrix $\mb X$. Is it possible to identify $\mb A$ and $\mb X$, up to scaling and permutation. If we assume that the rows of $\mb X$ are sampled from independent random sources, classical, general results in the literature on Independent Component Analysis imply that the problem is solvable in the large sample limit \cite{Comon94-SP}. If we instead assume that the columns of $\mb X$ each have at most $k$ nonzero entries, and that for each possible pattern of nonzeros, we have observed $k+1$ nondegenerate samples $\mb y_j$, the problem is again well-posed \cite{Aharon06,Georgiev2005-TNN}. This suggests a sample requirement of $p \ge (k+1)\binom{n}{k}$. We ask: is this large number necessary? Or could it be that the desired factorization is unique\footnote{Of course, for some applications, weaker notions than uniqueness may be of interest. For example, Vainsencher et.\ al.\ \cite{Vainsencher2011-COLT} give generalization bounds for a learned dictionary $\hat{\mb A}$. Compared to the results mentioned above, these bounds depend much more gracefully on the dimension and sparsity level. However, they do not directly imply that the ``true'' dictionary $\mb A$ is unique, or that it can be recovered by an efficient algorithm.} even with more realistic sample sizes?

Second, suppose that we know that the problem is well-posed. Can it be solved efficiently? This question has been vigorously investigated by many authors, starting from seminal work of Olshausen and Field \cite{Olshausen1996-Nature}, and continuing with the development of alternating directions methods such as the Method of Optimal Directions (MOD) \cite{Engan1999-ICASSP}, K-SVD \cite{Aharon2006-TSP}, and more recent, scalable variants \cite{Mairal09}.  This dominant approach to dictionary learning exploits the fact that the constraint $\mb Y = \mb A \mb X$ is bilinear. Because the problem is nonconvex, spurious local minima are a concern in practice, and even in the cases where the algorithms perform well empirically, providing global theoretical guarantees would be a daunting task. Even the local properties of the problem have only recently begun to be studied carefully. For example, \cite{Gribonval10,Wright11} have shown that under certain natural random models for $\mb X$, the desired solution will be a local minimum of the objective function with high probability. However, these results do not guarantee correct recovery by any efficient algorithm.

In this work, we contribute to the understanding of both of these questions in the case when $\mb A$ is square and nonsingular. We prove that $O (n\log n)$ samples are sufficient to uniquely determine the decomposition with high probability, under the assumption $\mb X$ is generated by a Bernoulli-Subgaussian process. 

Our argument for uniqueness suggests a new, efficient dictionary
learning algorithm, which we call Exact Recovery of Sparsely-Used
Dictionaries (ER-SpUD). This algorithm solves a sequence of linear
programs with varying constraints. We prove that under the
aforementioned assumptions, the algorithm exactly recovers $\mb A$ and
$\mb X$ with high probability. This result holds when the expected
number of nonzero elements in each column of $\mb X$ is at most $O (\sqrt{n})$
and the number of samples  $p$ is at least $\Omega (n^2\log^2 n)$. 
To the best of our knowledge, this result
is the first to demonstrate an efficient algorithm for dictionary
learning with provable guarantees.

Moreover, we prove that this result is tight to within a $\log $ factor: when the expected number of nonzeros in each column is $\Omega( \sqrt{ n\log n} )$,  algorithms of this style fail with high probability.

Our algorithm is related to previous proposals by Zibulevsky and Pearlmutter \cite{Zibulevsky2001-NC} (for source separation) and Gottlieb and Neylon \cite{Gottlieb10} (for dictionary learning), but involves several new techniques that seem to be important for obtaining provable correct recovery -- in particular, the use of sample vectors in the constraints. We will describe these differences more clearly in Section \ref{sec:alg}, after introducing our approach. Other related recent proposals include \cite{Plumbley2007-ICA,Jaillet2010-ICASSP}.

The remainder of this paper is organized as follows. In Section \ref{sec:model}, we fix our model. Section \ref{sec:uniq} discusses situations in which this problem is well-posed. Building on the intuition developed in this section, Section \ref{sec:alg} introduces the ER-SpUD algorithm for dictionary recovery. In Section \ref{sec:main}, we introduce our main theoretical results, which characterize the regime in which ER-SpUD performs correctly. Section \ref{sec:sketch} describes the key steps in our analysis. Technical lemmas and proofs are sketched; for full details please see the full version.  Finally, in Section \ref{sec:exp} we perform experiments corroborating our theory and suggesting the utility of our approach. 

\section{Notation} 
We write $\norm{\mb v}_{p}$ for the standard $\ell^{p}$ norm of a vector $\mb v$,
  and we write $\norm{\mb M}_{p}$ for the induced operator norm on a matrix $\mb M$.
$\norm{\mb v}_{0}$ denotes the number of non-zero entries in $\mb v$.
We denote the Hadamard (point-wise) product by $\odot$.
$[n]$ denotes the first $n$ positive integers, $\{1,2,\dots, n\}$. 
For a set of indices $I$, we let $\mb P_{I}$ denote the projection matrix onto
  the subspace of vectors supported on indices $I$, zeroing out the other coordinates.
For a matrix $\mb X$ and a set of indices $J$, we let $\mb X_{J}$ ($\mb X^{J}$) denote the submatrix
  containing just the rows (columns) indexed by $J$.
We write the standard basis vector that is non-zero in coordinate $i$ as $\mb e_{i}$.
For a matrix $\mb X$ we let $\mathrm{row}(\mb X)$ denote the span of its rows.
For a set $S$, $|S|$ is its cardinality.

\section{The Probabilistic Models}\label{sec:model}
We analyze the dictionary learning problem under the assumption that $\mb A$ is
  an  arbitrary nonsingular $n$-by-$n$ matrix, and $\mb X$ is a random sparse $n$-by-$p$ that follows the following probabilistic model:

\begin{definition} We say that $\mb X$ satisfies the {\em Bernoulli-Subgaussian model} with parameter $\theta$ if $\mb X = \mb \Omega \odot \mb R$, where $\mb \Omega$ is an iid $\mathrm{Bernoulli}(\theta)$ matrix, and $\mb R$ is an independent random matrix whose entries are iid symmetric random variables with
\begin{equation}\label{eqn:moments}
\mu \doteq \E\left[ \, |R_{ij}| \, \right] \in [1/10, 1], \quad \E\left[ \, R_{ij}^2 \, \right] \;\le\; 1.
\end{equation}
and
\begin{equation}\label{eqn:tail}
\P\left[ \, | R_{ij} | > t \, \right] \;<\; 2 \exp\left( - \frac{t^2}{2} \right) \quad \forall \, t \ge 0.
\end{equation}
\end{definition}

This model includes a number of special cases of interest -- e.g., standard Gaussians and Rademachers. The constant $1/10$ is not essential to our arguments and is chosen merely for convenience. The subgaussian tail inequality \eqref{eqn:tail} implies a number of useful concentration properties. In particular, if $x_1 \dots x_N$ are independent, random variables satisfying an inequality of the form \eqref{eqn:tail}, then 
\begin{equation} \label{eqn:sg-conc}
\P\left[ \left| \sum_{i=1}^N x_i  - \E\left[ \sum_{i=1}^N x_i \right] \right | > t \right] < 2 \exp\left( - \frac{t^2}{2 N} \right).
\end{equation}
We will occasionally refer to the following special case of the Bernoulli-Subgaussian model:
\begin{definition} We say that $\mb X$ satisfies the {\em Bernoulli-Gaussian model} with parameter $\theta$ if $\mb X = \mb \Omega \odot \mb R$, where $\mb \Omega$ is an iid $\mathrm{Bernoulli}(\theta)$ matrix, and $\mb R$ is an independent random matrix whose entries are iid $\mc N(0,1)$. 
\end{definition}

\section{When is the Factorization Unique?} \label{sec:uniq}

At first glance, it seems the number of samples $p$
required to identify $\mb A$ could be quite large. For example, Aharon \emph{et.\ al.} view the given data matrix $\mb X$ as having sparse columns, each with at most $k$ nonzero entries.
If the given samples $\mb y_j = \mb A \mb x_j$ lie on an arrangement of $\binom{n}{k}$ $k$-dimensional subspaces
$\mathrm{range}(\mb A_I)$, corresponding to possible support sets $I$, $\mb A$ is identifiable. 

On the other hand, the most immediate lower bound on the number of
samples required comes from the simple fact that to recover $\mb A$ we
need to see at least one linear combination involving each of its
columns. The ``coupon collection'' phenomenon tells us that 
$p = \Omega(\frac{1}{\theta}\log n )$ samples are required for this to occur with
constant probability, where $\theta$ is the probability that an element
$\mb X_{ij}$ is nonzero.  When $\theta$ is as small as $O (1/n)$, this means $p$ must
be at least proportional to $n \log n$. Our next result shows that, in
fact, this lower bound is tight -- the problem becomes well-posed once
we have observed $c n \log n$ samples.

\begin{theorem}[Uniqueness]\label{thm:uniq}
Suppose that $\mb X = \mb \Omega \odot \mb R$ follows the Bernoulli-Subgaussian model, and $\P[ R_{ij} = 0 ] = 0$. Then if $1/n \le \theta \le 1/C$ and $p>Cn\log n$, with probability at least $1-  C' n \exp\{-c \theta p\}$ the following holds:
\begin{quote}
For
any alternative factorization $\mb Y = \mb A' \mb X'$ such that $\max_i \| \mb e_i^T \mb X' \|_0 \;\le\; \max_i \| \mb e_i^T \mb X \|_0$, we have $\mb A' = \mb A \mb \Pi \mb \Lambda$ and $\mb X' = \mb \Lambda^{-1} \mb \Pi^T \mb X$, for some permutation matrix $\mb \Pi$ and nonsingular diagonal matrix $\mb \Lambda$. 
\end{quote}
Above, $c$, $C$, $C'$ are absolute constants. 
\end{theorem}

\subsection{Sketch of Proof} 
Rather than looking at the problem as one of trying to
recover the sparse columns of $\mb X$, we instead try to recover the
sparse rows. 
As $\mb X$ has full row rank with very high probability, the following lemma
  tells us that for any other factorization the row spaces of $\mb X$, $\mb Y$ and $\mb X'$
  are likely the same.

\begin{lemma}\label{lem:uniq_2}
If $\mathrm{rank}(\mb X)=n$, $\mb A$ is nonsingular, and $\mb Y$ can be decomposed into $\mb Y = \mb A' \mb X'$, then the row spaces of $\mb X'$, $\mb X$, and $\mb Y$ are the same. 
\end{lemma}

We will prove that the sparsest vectors in the row-span of $\mb Y$ are the
rows of $\mb X$. As any other factorization $\mb Y = \mb A' \mb X'$ will have the same
row-span, all of the rows of $\mb X'$ will lie in the row-span of
$\mb Y$. This will tell us that they can only be sparse if they are in
fact rows of $\mb X$. This is reasonable, since if distinct rows of $\mb
X$ have nearly disjoint patterns of nonzeros, taking linear
combinations of them will increase the number of nonzero entries.

\begin{lemma}\label{lem:uniqueMore}
Let $\mb \Omega$ be an $n$-by-$p$ $\mathrm{Bernoulli}(\theta)$ matrix  with $1/n < \theta < 1/4$. For each set $S \subseteq [n]$, let $T_{S} \subseteq [p]$ be the indices of the columns of $\mb \Omega$ that have at least one non-zero entry in some row indexed by $S$.
\begin{enumerate}
\item [a.] For every set $S$ of size $2$, 
\[
  \prob{{\sizeof{T_{S}} \leq (4/3) \theta p }} \;\leq\; \exp \left(- \frac{\theta p}{108} \right).
\]

\item [b.] For every set $S$ of size $\sigma $ with $3 \leq \sigma \leq 1/\theta  $
\[
  \prob{{\sizeof{T_{S}} \leq (3 \sigma / 8) \theta p }} \;\leq\; \exp \left( - \frac{\sigma \theta p}{ 64 } \right).
\]
\item [c.] For every set $S$ of size $\sigma $ with $1/\theta  \leq \sigma$,
\[
  \prob{{\sizeof{T_{S}} \leq (1-1/e)  p /2}} \;\leq\; \exp \left(-\frac{ (1-1/e)p}{8} \right).
\]
\end{enumerate}
\end{lemma}

Lemma \ref{lem:uniqueMore} says that every subset of at least two rows
of $\mb X$ is likely to be supported on many more than $\theta p$ columns,
which is larger than the expected number of nonzeros $\theta p$ in
any particular row of $\mb X$. We show that for any vector $\mb \alpha \in \Reals{n}$
  with support $S$ of size at least $2$, it is unlikely
  that $\mb \alpha^{T} \mb X$ is supported on many fewer columns than
  are in $T_{S}$. In the next lemma, we call a vector $\mb \alpha$ fully dense if all of its entries are nonzero. 

\begin{lemma}\label{lem:nnzRademacher}
For $t > 200 s$,  let $\mb \Omega\in \{0,1\}^{s \times t}$ be any binary matrix with at least
one nonzero in each column.  
Let $\mb R \in \Re^{s \times t}$ be a random matrix whose entries are iid symmetric random variables, with $\prob{R_{ij} = 0} = 0$, and let $\mb U = \mb \Omega \odot \mb R$. Then, the probability that 
  there exists a fully-dense vector $\mb \alpha$
  for which
  $\norm{\mb \alpha^{T} \mb U  }_{0} \leq t/5$
  is at most $2^{-t/25}$.

\end{lemma}

Combining Lemmas \ref{lem:uniqueMore} and \ref{lem:nnzRademacher}, we prove the following.  

\begin{lemma}\label{lem:uniqueLastGR} If $\mb X = \mb \Omega \odot \mb R$ follows the Bernoulli-Subgaussian model, with $\P[ R_{ij} = 0 ] = 0$,  $1/n < \theta < 1/C$ and $p > C n \log n$, then the probability that there is a vector $\mb \alpha$ with support of size larger than $1$ for which \[\norm{\mb \alpha^{T} \mb X}_{0} \;\leq\; (11/9) \, \theta p\] is at most $\exp (-c \theta p)$. Here, $C,c$ are numerical constants. 
\end{lemma}

Theorem~\ref{thm:uniq} follows from Lemmas~\ref{lem:uniq_2} and ~\ref{lem:uniqueLastGR} and the observation that with high probability each of the rows of $\mb X$ has at most $(10/9) \, \theta p$ nonzeros. We give a formal proof of Theorem \ref{thm:uniq} and its  supporting lemmas in Appendix \ref{app:unique-proof}.

\section{Exact Recovery}\label{sec:alg}

Theorem \ref{thm:uniq} suggests that we can recover $\mb X$ by looking
for sparse vectors in the row space of $\mb Y$. Any vector in this
space can be generated by taking a linear combination $\mb w^T \mb Y$
of the rows of $\mb Y$ (here, $\mb w^T$ denotes the vector
transpose). We arrive at the optimization problem
$$\text{minimize} \;\| \mb w^T \mb Y \|_0 \quad \text{subject to}
\quad \mb w \ne \mb 0.$$ Theorem \ref{thm:uniq} implies that any
solution to this problem must satisfy $\mb w^T \mb Y = \lambda \mb
e_j^T \mb X$ for some $j \in [n]$, $\lambda \ne 0$. Unfortunately,
both the objective and constraint are nonconvex. 
We therefore replace
  the $\ell^0$ norm with its convex envelope, the $\ell^1$ norm, and
  prevent $\mb w$ from being the zero vector by constraining it to
 lie in an affine hyperplane $\{ \mb r^T \mb w = 1 \}$. This gives a
  linear programming problem of the form
\begin{equation} \label{eqn:main-opt}
\text{minimize} \;\| \mb w^T \mb Y \|_1 \quad \text{subject to} \quad \mb r^T \mb w = 1.
\end{equation}
We will prove that this linear program is likely to produce rows of $\mb X$
  when we choose $\mb r$ to be a column or a sum of two columns of $\mb Y$.

\subsection{Intuition} 

To gain more insight into the optimization problem
\eqref{eqn:main-opt}, we consider for analysis an equivalent problem, under the change
of variables $\mb z = \mb A^T \mb w$, $\mb b = \mb A^{-1} \mb r$:
\begin{equation} \label{eqn:main-opt-X}
\text{minimize} \; \| \mb z^T \mb X \|_1 \quad \text{subject to} \quad \mb b^T \mb z = 1.
\end{equation}
When we choose $\mb r$ to be a column of $\mb Y$,
  $\mb b$ becomes a column of $\mb X$.
While we do not know $\mb A$ or $\mb X$ and so cannot directly solve problem
  \eqref{eqn:main-opt-X}, it is equivalent to problem \eqref{eqn:main-opt}:
   \eqref{eqn:main-opt} recovers a row of $\mb X$ if and only
  if the solution to \eqref{eqn:main-opt-X} is a scaled multiple of a
  standard basis vector: $\mb z_\star = \lambda \mb e_j$, for some $j$,
  $\lambda$.

To get some insight into why this might occur, consider what would
happen if $\mb X$ exactly preserved the $\ell_{1}$ norm: i.e., if $\| \mb z^T \mb
X \|_1 = c \| \mb z \|_1$ for all $\mb z$ for some constant $c$.  
The solution to
  \eqref{eqn:main-opt-X} would just be the vector $\mb z$ of smallest $\ell^1$
  norm satisfying $\mb b^T \mb z = 1$, which would be  $\mb e_{j_\star} /
  b_{j_\star}$, where 
  $j_\star$ is the index of the  element of
  $\mb b = \mb A^{-1} \mb r$ of largest magnitude. The algorithm would simply
extract the row of $\mb X$ that is most ``preferred'' by $\mb b$!

Under the random coefficient models considered here, $\mb X$ {\em approximately} preserves the $\ell_{1}$ norm, but does not exactly preserve it \cite{Matousek08}. Our algorithm can tolerate this approximation if
  the largest element of $\mb b$
  is significantly larger than the other elements.
In this case
  we can still apply the above argument to show that
  \eqref{eqn:main-opt-X} will recover the $j_{\star}$-th row of $\mb X$.
In particular, if
  we let $|\mb b|_{(1)} \ge |\mb b|_{(2)} \ge \dots \ge |\mb b|_{(n)}$
  be the absolute values of the entries of $\mb b$ in decreasing order,
  we will require both $|\mb b|_{(2)} / |\mb b|_{(1)} < 1 - \gamma$
  and that the total number of nonzeros in $\mb b$ is at most $c / \theta$. The gap $\gamma$ determines fraction $\theta$ of nonzeros that the algorithm can tolerate. 

If the nonzero entries of $\mb X$ are Gaussian, then when we choose $\mb r$ to be a column of $\mb Y$ (and thus
  $\mb b = \mb A^{-1} \mb r$ to be a column of $\mb X$), properties of the order statistics
  of Gaussian random vectors imply that our requirements are probably met. In other coefficient models, the gap $\gamma$ may not be so prominent. For example, if the nonzeros of $\mb X$ are Rademacher (iid $\pm 1$), there is no gap whatsoever between the magnitudes of the largest  and second-largest elements. For this reason, we instead choose $\mb r$ to be the sum of two columns of $\mb Y$
  and thus $\mb b$ to be the sum of two columns of $\mb X$.
When $\theta < 1/\sqrt{n}$, there is a reasonable chance that the support of these
  two columns overlap in exactly one element, in which case we obtain a gap
  between the magnitudes of the largest two elements in the sum.
This modification also provides improvements in the Bernoulli-Gaussian model.

\subsection{The Algorithms}

Our algorithms are divided into two stages. 
In the first stage, we collect many potential rows of $\mb X$ by solving problems of the
  form \eqref{eqn:main-opt}. 
In the simpler Algorithm \textbf{ER-SpUD(SC)} (``single column''),
  we do this by using each column of $\mb Y$ as the constraint vector $\mb r$
  in the optimization.
In the slightly better Algorithm \textbf{ER-SpUD(DC)} (``double column''),
  we pair up all the columns of $\mb Y$ and then substitue the sum
  of each pair for $\mb r$.
In the second stage, we use a greedy algorithm (Algorithm \textbf{Greedy}) to select
  a subset of $n$ of the rows produced.
In particular, we choose a linearly independent subset among those
  with the fewest non-zero elements.
From the proof of the uniqueness of the decomposition, we know with high
  probability that the rows of $\mb X$ are the sparsest $n$ vectors in
  $\mathrm{row}(\mb Y)$.
Moreover, for $p \geq \Omega ( n \log  n)$, Theorems~\ref{thm:correct} and \ref{thm:correct-tc},
  along with the coupon collection phenomenon, tell us that a scaled multiple 
  of each of the rows of $\mb X$ is returned by the first phase of our algorithm, with high probability.\footnote{Preconditioning by setting $\mb Y_p=(\mb Y \mb Y^T)^{-1/2}\mb Y$ helps in simulation, while our analysis does not require $\mb A$ to be well conditioned.}

\newenvironment{fminipage}%
 {\begin{Sbox}\begin{minipage}}%
 {\end{minipage}\end{Sbox}\fbox{\TheSbox}}

\newenvironment{algbox}[0]{\vskip 0.2in
\noindent
\begin{fminipage}{6.3in}
}{
\end{fminipage}
\vskip 0.2in
}
\begin{algbox}
\noindent \textbf{ER-SpUD(SC):}\texttt{ Exact Recovery of Sparsely-Used Dictionaries using single columns of $\mb Y$ as constraint vectors.}\label{alg:spud-sc}
\begin{enumerate}
\item[] For $j=1\dots p$
\begin{enumerate}
\item []Solve
  $\min_{\mb w} \; \| \mb w^T \mb Y\|_1 \text{ subject to }  (\mb Y \mb e_j)^T \mb w = 1,$ and set $\mb  s_{j} = \mb w^{T} \mb Y$. 
\end{enumerate}
\end{enumerate}
\end{algbox}

\begin{algbox}
\noindent \textbf{ER-SpUD(DC):}\texttt{ Exact Recovery of Sparsely-Used Dictionaries using the sum of two columns of $Y$ as constraint vectors.}\label{alg:spud-dc}
\begin{enumerate}
\item Randomly pair the columns of $\mb Y$ into $p/2$ groups $g_j=\{\mb Y\mb e_{j_{1}},\mb Y\mb e_{j_{2}}\}$.
\item For $j=1\dots p/2$
\begin{enumerate}
\item [] Let $\mb r_j=\mb Y\mb e_{j_{1}}+\mb Y\mb e_{j_{2}}$, where $g_j = \setof{\mb Y\mb e_{j_{1}}, \mb Y \mb e_{j_{2}}}  $.
\item [] Solve $\min_{\mb w} \; \| \mb w^T \mb Y\|_1  \text{ subject to }  \mb r_j^T \mb w = 1,$
   and set $\mb  s_{j} = \mb w^{T} \mb Y$. 
\end{enumerate}
\end{enumerate}
\end{algbox}
\vspace{-0.2in}
\begin{algbox}
\noindent \textbf{Greedy:}\texttt{ A Greedy Algorithm to Reconstruct $\mb X$ and $\mb A$.}\label{alg:spud-greedy}
\begin{enumerate}
\item \textbf{REQUIRE:} $\mc S = \{ \mb s_1, \dots, \mb s_T \} \subset \Re^p$. 
\item For $i=1\dots n$
\begin{enumerate}
\item [] REPEAT
\begin{enumerate}
\item[] $l \gets \arg\min_{\mb s_l \in \mc S}\| \mb s_l \|_0$, breaking ties arbitrarily
\item[] $\mb x_i=\mb s_l$
\item[] $\mc S=\mc S \backslash \{\mb s_l\}$
\end{enumerate}
\item[] \textbf{UNTIL} \texttt{rank([$\mb x_1,\dots, \mb x_i$])$=i$}
\end{enumerate}
\item Set $\mb X=[\mb x_1,\dots, \mb x_n]^T$, and $\mb A = \mb Y \mb Y^T(\mb X \mb Y^T)^{-1}$.
\end{enumerate}
\end{algbox}

\paragraph{Comparison to Previous Work.} The idea of seeking the rows of $\mb X$ sequentially, by looking for sparse vectors in $\mathrm{row}(\mb Y)$, is not new {\em per se}. For example, in \cite{Zibulevsky2001-NC}, Zibulevsky and Pearlmutter suggested solving a sequence of optimization problems of the form
$$\text{minimize} \;\; \| \mb w^T \mb Y \|_1 \quad \text{subject to} \quad \| \mb w \|_2^2 \ge 1.$$
However, the non-convex constraint in this problem makes it difficult to solve.
In more recent work, Gottlieb and Neylon \cite{Gottlieb10} suggested using linear constraints as in \eqref{eqn:main-opt}, but choosing $\mb r$ from the standard basis vectors $\mb e_1 \dots \mb e_n$. 

The difference between our algorithm and that of Gottlieb and Neylon---the 
  use of columns of the sample
  matrix $\mb Y$ as linear constraints instead of elementary unit vectors,
  is crucial to the functioning of our algorithm
(simulations of their Sparsest Independent Vector algorithm are reported below).
In fact, there are simple examples of orthonormal matrices
   $\mb A$ for which the algorithm of
\cite{Gottlieb10} provably fails, whereas Algorithm \textbf{ER-SpUD(SC)}
succeeds with high probability. One concrete example of this is a
Hadamard matrix: in this case, the entries of $\mb b = \mb A^{-1} \mb
e_j$ all have exactly the same magnitude, and \cite{Gottlieb10} fails
because the gap between $|\mb b|_{(1)}$ and $|\mb b|_{(2)}$ is
zero when $\mb r$ is chosen to be an elementary unit vector.
In this situation, Algorithm \textbf{ER-SpUD(DC)} still succeeds
with high probability.

\section{Main Theoretical Results}\label{sec:main}

The intuitive explanations in the previous section can be made rigorous. In particular, under our random models, we can prove that when the number of samples is reasonably large compared to the dimension, (say $p \sim n^2 \log^2 n$), with high probability in $\mb X$ the algorithm will succeed.
We conjecture it is possible to decrease the dependency on $p$ to $O (n \log n)$.

\begin{theorem}[Correct recovery (single-column)] \label{thm:correct}
Suppose $\mb X$ is $\mathrm{Bernoulli}(\theta)\mathrm{-Gaussian}$. 
Then provided $p > c_{1} n^2 \log^2 n$, and 
\begin{equation}
\frac{2}{n}\;\leq\; \theta \;\le\; \frac{\alpha }{\sqrt{n \log n}},
\end{equation}
with probability at least $1-c_f p^{-10}$,
the Algorithm \textbf{ER-SpUD(SC)} recovers all $n$ rows of $\mb X$. That is, all $n$ rows of $\mb X$ are included in the $p$ potential vectors $\mb w_1^T \mb Y,\dots, \mb w_p^T \mb Y$. Above, $c_1$, $\alpha$ and $c_f$ are positive numerical constants. 
\end{theorem}

The upper bound of $\alpha  / \sqrt{n \log n}$ on $\theta $ has two sources:
  an upper bound of $\alpha / \sqrt{n}$ is imposed by the requirement that $\mb b$ be sparse.
An additional factor of $\sqrt{\log n}$ comes from the need for a gap between $|\mb b|_{(1)}$ and $|\mb b|_{(2)}$ of the $k$ i.i.d.\ Gaussian random variables.  
On the other hand, using the sum of two columns of $\mb Y$ as $\mb r$ can save the factor of $\log n$ in the requirement on $\theta$ since the ``collision'' of non-zero entries in the two columns of $\mb X$ creates a larger gap between $|\mb b|_{(1)}$ and $|\mb b|_{(2)}$. More importantly, the resulting algorithm is less dependent on the magnitudes of the nonzero elements in $\mb X$. The  algorithm using a single column exploited the fact that there exists a reasonable gap between $|\mb b|_{(1)}$ and $|\mb b|_{(2)}$, whereas the two-column variant \textbf{ER-SpUD(DC)} succeeds even if the nonzeros all have the same magnitude. 

\begin{theorem}[Correct recovery (two-column)] \label{thm:correct-tc}
Suppose $\mb X$ follows the Bernoulli-Subgaussian model. 
Then provided $p > c_{1} n^2 \log^2 n$, and 
\begin{equation}
\frac{2}{n}\;\leq\; \theta \;\le\; \frac{\alpha }{\sqrt{n}},
\end{equation}
with probability at least $1-c_f p^{-10}$,
the Algorithm \textbf{ER-SpUD(SC)} recovers all $n$ rows of $\mb X$. That is, all $n$ rows of $\mb X$ are included in the $p$ potential vectors $\mb w_1^T \mb Y,\dots, \mb w_p^T \mb Y$. Above, $c_1$, $\alpha$ and $c_f$ are positive numerical constants. 
\end{theorem}

Hence, as we choose $p$ to grow faster than $n^2\log^2 n$, the algorithm will succeed with probability approaching one. That the algorithm succeeds is interesting, perhaps even unexpected. There is potentially a great deal of symmetry in the problem -- all of the rows of $\mb X$ might have similar $\ell^1$-norm. The vectors $\mb r$ break this symmetry, preferring one particular solution at each step, at least in the regime where $\mb X$ is sparse. To be precise, the expected number of nonzero entries in each column must  be bounded by $\sqrt{ n\log n}$.

It is natural to wonder whether this is an artifact of the analysis, or whether such a bound is necessary. We can prove that for Algorithm \textbf{ER-SpUD(SC)}, the sparsity demands in Theorem \ref{thm:correct-tc} cannot be improved by more than a factor of $\sqrt{ \log n}$. Consider the optimization problem \eqref{eqn:main-opt-X}. One can show that for each $i$, $\| \mb e_i^T \mb X \|_1 \approx \theta p$. Hence, if we set $\mb z = \mb e_{j_\star} / b_{j_\star}$, where $j_\star$ is the index of the largest element of $\mb b$ in magnitude, then $$\| \mb z^T \mb X \|_1 \;=\; \frac{\| \mb e_{j_\star}^T \mb X \|_1}{\| \mb b \|_\infty} \;\approx\; C\frac{\theta p}{\sqrt{\log n}}.$$
If we consider the alternative solution $\mb v = \mathrm{sign}(\mb b) / \| \mb b \|_1$, a calculation shows that 
$$\| \mb v^T \mb X \|_1 \;\approx\; C' p /\sqrt{ n}.$$
Hence, if $\theta > c \sqrt{\log n / n}$ for sufficiently large $c$, the second solution will have smaller objective function. These calculations are carried through rigorously in the full version, giving:

\begin{theorem} \label{thm:ub}
If $\mb X$ follows the Bernoulli-Subgaussian model with
$$\theta \ge \sqrt{\frac{\beta\log n}{n}} ,$$
with $\beta \ge \beta_0$, and the number of samples $p>c n \log n$, then the probability that solving the optimization problem
\begin{equation}
\text{minimize} \;\| \mb w^T \mb Y \|_1 \quad \text{subject to} \quad \mb r^T \mb w = 1
\end{equation}
with $\mb r = \mb Y \mb e_j$ recovers one of the rows of $\mb X$ is at most
\begin{equation}
\exp\left(- c p \right) + \exp\left( - 3 \beta \sqrt{n \log n} \right) + 4 \exp\left( - c' \theta p + \log n \right)  + 2 n^{1-5\beta}
\end{equation}
above, $\beta_0, c, c'$ are positive numerical constants. 
\end{theorem}

This implies that the result in Theorem \ref{thm:correct} is nearly the best possible for this algorithm, at least in terms of its demands on $\theta$. A nearly identical result can be proved with $\mb r = \mb Y \mb e_i + \mb Y \mb e_j$ the sum of two rows, implying that similar limitations apply to the two-column version of the algorithm. 

\section{Sketch of the Analysis} \label{sec:sketch}

In this section, we sketch the arguments used to prove Theorem
\ref{thm:correct}. The proof of Theorem \ref{thm:correct-tc} is
similar. The arguments for both of these results are carried through rigorously in 
Appendix \ref{sec:recovery-proof}. At a high level, our argument follows the
intuition of Section \ref{sec:alg}, using the order statistics and the
sparsity property of $\mb b$ to argue that the solution must recover a
row of $\mb X$. 
We say that a vector is $k$-sparse if it has at most $k$ non-zero entries.
Our goal is to show that $\mb z_{\star}$
 is $1$-sparse.
We find it convenient to do this in two steps.

We first argue that the solution $\mb z_\star$ to \eqref{eqn:main-opt-X} must be supported on indices that are non-zero in $\mb b$, so $\mb z$ is at least as sparse as $\mb b$, say $\sqrt{n}$-sparse in our case. Using this result, we restrict our attention to a submatrix of $\sqrt{n}$ rows of $\mb X$, and prove that for this restricted problem, when the gap $1-|\mb b|_{(2)} / |\mb b |_{(1)}$ is large enough, the solution $\mb z_\star$ is in fact $1$-sparse, and we recover a row of $\mb X$. 

\paragraph{Proof solution is sparse.}  
We first
show that with high probability, the solution $\mb z_\star$ to \eqref{eqn:main-opt-X} is
 supported only on the non-zero indices in $\mb b$. 
Let $J$ denote the
indices of the $s$ non-zero entries of $|\mb b|$, and let $S = \{ j
\mid \mb X_{J,j} \ne \mb 0 \} \subset [p]$, i.e., the indices of the
nonzero columns in $X_J$. Let $\mb z_0 = \mb P_J \mb z_\star$ be the restriction of $\mb z_\star$ to those coordinates indexed by $J$, 
and $\mb z_1 = \mb z_\star - \mb z_0$. By definition, $\mb z_0$ is supported
on $J$ and $\mb z_1$ on $J^c$. 
Moreover, $\mb z_{0}$ is feasible for Problem~\eqref{eqn:main-opt-X}.  
We will show that it has at least as  low an objective function value as
  $\mb z_{\star}$, and thus conclude that $\mb z_{1}$ must be zero.
Write
\begin{align}
\| \mb z_\star^T \mb X \|_1 \;&=\; \| \mb z_\star^T \mb X^{S} \|_1 + \| \mb z_\star^T \mb X^{S^c} \|_1 \;\ge\; \| \mb z_0^T \mb X^{S} \|_1 - \| \mb z_1^T \mb X^{S} \|_1 + \| \mb z_1^T \mb X^{S^c} \|_1\nonumber\\
&=\| \mb z_0^T \mb X \|_1 - 2\| \mb z_1^T \mb X^{S} \|_1 + \| \mb z_1^T \mb X\|_1,
\end{align}
where we have used the triangle inequality and the fact that $\mb z_0^T \mb X^{S^c} = \mb 0$. 
In expectation we have that
\begin{equation} \label{eqn:z-opt-bound}
\| \mb z_\star^T \mb X \|_1 \;\ge\; \| \mb z_0^T \mb X \|_1 +(p-2|S|)\mathbb{E}[\|\mb z_1^T \mb X \|_1] \;\ge\; \| \mb z_0^T \mb X \|_1 +c(p-2|S|)\sqrt{\theta /n}\|\mb z_1\|_1,
\end{equation}
where the last inequality requires $\theta n\geq 2$.

So as long as $p-2|S|>0$, $\mb z_0$ has
lower expected objective value.
To prove that this happens with high probability, we first upper bound $|S|$ by
the number of nonzeros in $\mb X_J$, which in expectation is $\theta s
p$. As long as $p-2(1+\delta)\theta s p=p(1-c'\theta s)>0$, or
equivalently $s<c_s/\theta$ for some constant $c_s$, we have $\| \mb
z_\star^T \mb X \|_1>\| \mb z_0^T \mb X \|_1$.  In the following
lemma, we make this argument formal by proving concentration around
the expectation.

\begin{lemma}\label{lem:reduction-new} Suppose that $\mb X$ satisfies the Bernoulli-Subgaussian model. There exists a numerical constant $C > 0$, such that if 
$\theta n \ge 2$ and 
\begin{equation}
p > C n^2\log^2 n
\end{equation}
then with probability at least $1 - 3 p^{-10}$, the random matrix $\mb X$ has the following property:
 \begin{quote}
 {\bf (P1)} 
For every $\mb b$ satisfying $\|\mb b\|_0\leq 1/8 \theta$,
 any solution $\mb z_\star$ to the optimization problem 
 \begin{equation}
 \textrm{minimize} \;\| \mb z^T \mb X \|_1 \quad \textrm{subject to} \quad \mb b^T \mb z = 1
 \end{equation}
has $\mathrm{supp}(\mb z_\star) \subseteq \mathrm{supp}(\mb b)$.
 \end{quote}
\end{lemma}

Note in problem (\ref{eqn:main-opt-X}), $\mb b=\mb A^{-1} \mb r$. If
we choose $\mb r=\mb Y \mb e_i$, then $\mb b=\mb A^{-1} \mb Y \mb e_i= \mb X \mb e_i$, and
$\mathbb{E} [\| \mb b\|_0]=\theta n$. A Chernoff bound then tells us that with high probability
$\mb z_\star$ is supported on no more than $2\theta n$ entries, i.e.,
$s<2\theta n$. Thus as long as $2\theta n<c_{2}/\theta$, i.e.,
$\theta<c_\theta/\sqrt{n}$, we have $\|\mb z_\star\|_0<2\theta
n=c_\theta\sqrt{n}$.

\paragraph{The solution in $\mb X_J$:} If we restrict
our attention to the induced $s$-by-$p$ submatrix $\mb X_J$, we observe
that $\mb X_J$ is incredibly sparse -- most of the columns have at most one
nonzero entry. Arguing as we did in the first step, let
$j^\star$ denote the index of the largest entry of $|\mb b_J|$, and
let $S = \{ j \mid \mb X_J(j^\star,j) \ne \mb 0 \} \subset [p]$, i.e.,
the indices of the nonzero entries in the $j^\star$-th row of
$X_J$. Without loss of generality, let's assume $b_{j^\star}=1$. For
any $\mb z$, write $\mb z_0 = \mb P_{j^\star} \mb z$ and $\mb z_1 =
\mb z - \mb z_0$. Clearly $\mb z_0$ is supported on the $j^\star$-th
entry and $\mb z_1$ on the rest. As in the first step,
\begin{align} \label{eqn:reduced-split}
\| \mb z^T \mb X_J \|_1 &\geq \| \mb z_0^T \mb X_J \|_1 - 2\| \mb z_1^T \mb X_{J}^{S} \|_1 + \| \mb z_1^T \mb X_J\|_1.
\end{align}
By restricting our attention to $1$-sparse columns of $\mb X_J$,
 we prove that with high probability
\[
\| \mb z_1^T \mb X_J\|_1\geq \mu \theta p( 1 - s \theta )(1-\eps)^2\|\mb z_1\|_1.
\]
We prove that with high probability the second term of \eqref{eqn:reduced-split} satisfies
$$\| \mb z_1^T \mb X_{J,S} \|_1\leq (1+\epsilon)\mu \theta^2 p\|\mb z_1\|_1.$$ 
For the first term, we show 
$$\| \mb z_0^T \mb X_J \|_1\geq \|\mb e_{j^\star}^T\mb X_{J}\|_1-|\mb b_J^T\mb z_1|\|\mb X_{J}\|_1\geq \|\mb e_{j^\star}^T\mb X_{J}\|_1-|\mb b_J^T\mb z_1|(1+\epsilon)\mu \theta p.$$  
If $|\mb b|_{(2)} / |\mb b|_{(1)} < 1 - \gamma$, then $|\mb b_J^T \mb z_1|\leq ( 1 - \gamma )\|\mb z_1\|_1$.

In Lemma \ref{lem:extreme-sparse-new}, we combine these inequalities to show that if 
 if $\theta \le c \gamma / s$, then
\begin{align}
\|(\mb z_0+\mb z_1)^T \mb X_J\|_1&\geq  \|\mb e_{j^\star}^T\mb X_J\|_1+\mu\theta p \left(1- \gamma \right)\|\mb z_1\|_1.
\end{align}
Since $e_{j^\star}$ is a feasible solution to Problem \ref{eqn:main-opt-X} with a lower objective value as long as $\mb z_1\neq 0$, we know $e_{j^\star}$ is the only optimal solution. The following lemma makes this precise.

\begin{lemma}\label{lem:extreme-sparse-new} Suppose that $\mb X$ follows the Bernoulli-Subgaussian model. There exist positive numerical constants $c_1, c_2$ such that the following holds. For any $\gamma > 0$ and $s \in \mathbb{Z}_+$ such that $\theta s < \gamma/8$ and $p$ is sufficiently large:
$$p \;\ge\; \max\left\{\frac{c_1 s \log n}{\theta \gamma^2},n\right\} \quad \text{and} \quad \frac{p}{\log p} \;\ge\; \frac{c_2}{\theta \gamma^2},$$
then with probability at least $1 - 4 p^{-10}$, the random matrix $\mb X$ has the following property:
\begin{quote}
{\bf (P2)} For every $J \in \binom{[n]}{s}$ and every $\mb b \in \Re^s$ satisfying $|\mb b|_{(2)}/|\mb b|_{(1)} \;\le\; 1- \gamma$,
the solution to the restricted problem, 
\begin{align}
\textrm{minimize} \quad \|\mb z^T \mb X_{J,\ast} \|_1 \quad \textrm{subject to} \quad \mb b^T  \mb z = 1,\label{eqn:red-opt}
\end{align}
is unique, $1$-sparse, and is supported on the index of the largest entry of $\mb b$. 
\end{quote}
\end{lemma}

Once we know 
  that a column of $\mb Y$ provides us with a constant probability of recovering one row of $\mb X$,
 we know that we need only use $O (n \log n)$ columns to recover all the rows of $\mb X$
  with high probability. We give proofs of the above lemmas in Section \ref{sec:recovery-proof}. Section \ref{sec:recovery-proof-sc} shows how to put them together to prove Theorem \ref{thm:correct}.

\section{Simulations}\label{sec:exp}

\begin{algbox}
\noindent \textbf{ER-SpUD(proj):}\texttt{ Exact Recovery of Sparsely-Used Dictionaries with Iterative Projections.}\label{alg:spud-proj}
\begin{enumerate}
\item[] $\mc S \gets \{ 0 \} \subset \Re^n$.
\item[] For $i = 1 \dots n$
\begin{enumerate}
\item[] For $j=1\dots p$
\begin{enumerate}
\item []Find $\mb w_{ij} \;\in\; \arg\min_{\mb w} \; \| \mb w^* \mb Y\|_1 \quad \text{subject to} \quad \mb (Ye_j)^T \mb P_{\mc S^\perp} \mb w = 1.$ 
\end{enumerate}
\item[] $\mb w_i \gets \arg\min_{\mb w = \mb w_{i1}, \dots, \mb w_{iT}} \| \mb w^* \mb Y \|_0$, breaking ties arbitrarily.
\item[] $\mc S \gets \mc S \oplus \mathrm{span}(\mb w_i)$.
\end{enumerate}
\item[] $\mb X \gets \mb W \mb Y$.
\item[] $\mb A \gets \mb Y \mb Y^*( \mb X \mb Y^*)^{-1}$
\end{enumerate}
\end{algbox}

In this section we systematically evaluate our algorithm, and compare
it with the state-of-the-art dictionary learning algorithms, including
K-SVD \cite{Aharon2006-TSP}, online dictionary learning
\cite{Mairal09}, SIV \cite{Gottlieb10}, and the relative Newton method
for source separation \cite{Zibulevsky03}. The first two methods are
not limited to square dictionaries, while the final two methods, like
ours, exploit properties of the square case. The method of
\cite{Zibulevsky03} is similar in provenance to the incremental
nonconvex approach of \cite{Zibulevsky2001-NC}, but seeks to recover
all of the rows of $\mb X$ simultaneously, by seeking a local minimum
of a larger nonconvex problem. We found in the experiments that a slight variant of the greedy ER-SPUD algorithm, we call the ER-SPUD(proj), works even better than the greedy scheme.\footnote{Again, preconditioning by setting $\mb Y_p=(\mb Y \mb Y^T)^{-1/2}\mb Y$ helps in simulation.} And thus we also add its result to the comparison list. As our emphasis in this paper is mostly
on correctness of the solution, we modify the default settings of
these packages to obtain more accurate results (and hence a fairer
comparison). For K-SVD, we use high accuracy mode, and switch the
number of iterations from 10 to 30. Similarly, for relative Newton, we
allow 1,000 iterations. For online dictionary learning, we allow
1,000. We observed diminishing returns beyond these numbers. Since
K-SVD and online dictionary learning tend to get stuck at local optimum, for each trial
we restart K-SVD and Online learning algorithm 5 times with randomized
initializations and report the best performance. We measure accuracy
in terms of the relative error, after permutation-scale ambiguity has
been removed:
$$\tilde{\mathrm{re}}(\hat{\mb A},\mb A) \;\doteq\; \min_{\mb \Pi, \mb \Lambda} \|\hat{\mb A}\mb \Lambda \mb \Pi - \mb A \|_F/\|\mb A\|_F.\vspace{-1mm}$$

\paragraph{Phase transition graph.} 
In our experiments we have chosen $\mb A$ to be a an $n$-by-$n$
  matrix of independent 
  Gaussian random variables.
The coefficient matrix $\mb X$ is $n$-by-$p$, where $p = 5 n \log_e n$.
Each column of $\mb X$ has $k$ randomly chosen non-zero entries.
In our experiments we have varied $n$ between $10$ and $60$
  and $k$ between $1$ and $10$.
Figure \ref{fig:PhaseTrans}
  shows the results for each method,
  with the average relative error reported in greyscale.
White means zero error and black is $1$. The best performing algorithm is ER-SpUD with iterative projections, which solves almost all the cases except when $n=10$  and $k\geq 6$. For the other algorithm,
When $n$ is small, the relative
Newton method appears to be able to handle a denser $\mb X$, while as
$n$ grows large, the greedy ER-SpUD is more precise. In fact, empirically the
phase transition between success and failure for ER-SpUD is quite
sharp -- problems below the boundary are solved to high numerical
accuracy, while beyond the boundary the algorithm breaks down. In
contrast, both online dictionary learning and relative Newton exhibit
neither the same accuracy, nor the same sharp transition to failure --
even in the black region of the graph, they still return solutions
that are not completely wrong. The breakdown boundary of K-SVD is
clear compared to online learning and relative Newton. As an active
set algorithm, when it reaches a correct solution, the numerical
accuracy is quite high. However, in our simulations we observe that
both K-SVD and online learning may be trapped into a local optimum
even for relatively sparse problems.  \vspace{-.1in}

\begin{figure*}[t]
\begin{center}
\subfigure[ER-SpUD(SC)]{\includegraphics[height=3.8cm]{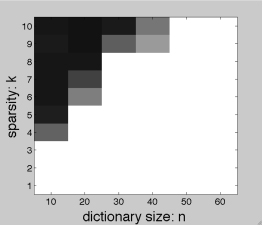}}
\subfigure[ER-SpUD(proj)]{\includegraphics[height=3.8cm]{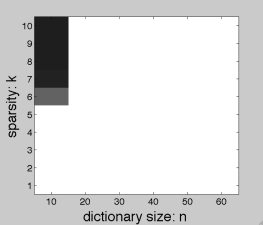}}
\subfigure[SIV]{\includegraphics[height=3.8cm]{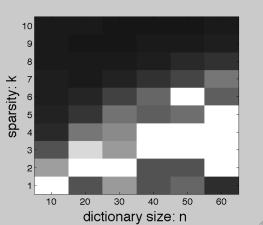}}\\ 
\subfigure[K-SVD ]{\includegraphics[height=3.8cm]{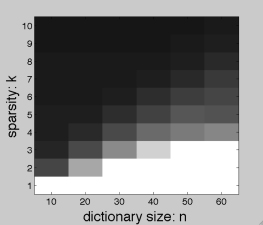}}
\subfigure[Online]{\includegraphics[height=3.8cm]{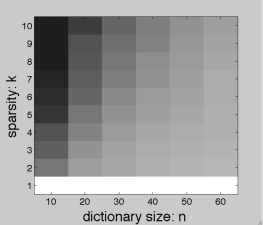}}
\subfigure[Rel.\ Newton]{\includegraphics[height=3.8cm]{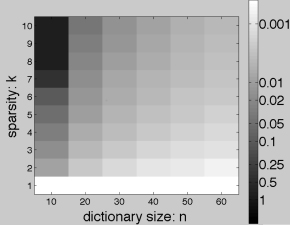}}
\end{center}
 \caption{\small Mean relative errors over 10 trials, with varying support $k$ (y-axis, increase from bottom to top) and basis size $n$(x-axis, increase from left to right). Here, $p=5 n\log_e n$. Our algorithm using a column of $Y$ as $r$ (ER-SpUD(SC)), our algorithm with iterative projections (ER-SpUD(proj)), SIV \cite{Gottlieb10}, K-SVD \cite{Aharon2006-TSP}, online dictionary learning \cite{Mairal09},  and the relative Newton method for source separation \cite{Zibulevsky03}.}
\label{fig:PhaseTrans}
\end{figure*}

\section{Discussion}\label{sec:conclu}

The main contribution of this work is a dictionary learning algorithm
with provable performance guarantees under a random coefficient
model. To our knowledge, this result is the first of its
kind. However, it has two clear limitations: the algorithm
requires that the reconstruction be exact, i.e., $\mb Y=\mb A\mb X$
and it requires $\mb A$ to be square.
It would be interesting to address both of these issues
  (see also \cite{Bach2008-TR} for
  investigation in this direction).
Finally, while our results pertain to
a specific coefficient model, our analysis generalizes to other
distributions. Seeking meaningful, deterministic assumptions on $\mb
X$ that will allow correct recovery is another interesting direction
for future work.

\acks{This material is based in part upon work supported by the National Science Foundation under Grant No. 0915487. JW also acknowledges support from Columbia University.}

\bibliography{RADL}

\appendix

\section{Proof of Uniqueness} \label{app:unique-proof}
In this section we prove our upper bound on the number of samples for which
  the decomposition of $\mb Y$ into $\mb A \mb X$ with sparse $\mb X$ is unique up to scaling 
  and permutation. We begin by recording the proofs of several lemmas from Section \ref{sec:uniq}. 


\subsection{Proof of Lemma \ref{lem:uniq_2}}
\begin{proof}
Since $\mathrm{rank}(\mb X)=n$, we know $\mathrm{rank}(\mb A')\geq \mathrm{rank}(\mb Y)= \mathrm{rank}(\mb A)=n$.
Since both $\mb A$ and $\mb A'$ are nonsingular, the row spaces of $\mb X'$ and $\mb X$ are the same as that of $\mb Y$. 
\end{proof}

\subsection{Proof of Lemma \ref{lem:uniqueMore}}
\begin{proof}
First consider sets $S$ of two rows.
The expected number of columns that have non-zero entries
  in at least one of these two rows is
\[
  p (1 - (1-\theta)^{2}) = p (2 \theta - \theta^{2}) \geq (3/2) p \theta ,
\]
for $\theta \leq 1/2$.
Part $a$ now follows from a Chernoff bound.

For part b, is $\sigma \ge 3$ and $\sigma \theta < 1$, we observe that for every $S$
\begin{align*}
\E \sizeof{T_{S}} \quad=\quad p-(1-\theta)^\sigma p \quad\geq\quad \left(\sigma \theta-(_2^\sigma)\theta^2 \right) p \quad=\quad \left( 1-\frac{\sigma -1}{2}\theta \right) \sigma\theta p \quad\geq\quad \frac{\sigma \theta p}{2},
\end{align*}
where the inequalities follow from $\sigma \theta\le1$.
Part $b$ now follows from a Chernoff bound.

For part c, if $\sigma \theta > 1$,
  for every $S$ of size $\sigma$ we have
\begin{align*}
\E \sizeof{T_{S}}& \geq (1-e^{-\sigma\theta }) p
 \geq (1-e^{-1})p.
\end{align*}
As before, the result follows from a Chernoff bound.
\end{proof}

\begin{definition}[fully dense vector] We call a vector $\mb \alpha\in \mathbb{R}^n$ fully dense if for all $i\in [n],~\alpha_i\neq 0$.
\end{definition}

\subsection{Proof of Lemma \ref{lem:uniqueLastGR}}

We use the following  theorem of Erd{\"o}s.
\begin{theorem}[\cite{Erdos45}]\label{thm:erdos}
For every $k \geq 2$ and nonzero real numbers $z_{1}, \dots , z_{k}$,
\[
  \prob{\sum_{i} z_{i} r_{i} = 0}
\;\leq\; 2^{-k} \binom{k}{\floor{k/2}} \;\leq\; 1/2,
\]
where each $r_{i}$ is chosen independently from $\pm 1$,
\end{theorem}

\begin{lemma}\label{lem:uniq_2.5R}
For $b > s$,
  let $\mb H \in \Re^{s \times b}$ be any matrix with at least
one nonzero in each column.  
Let $\mb R$ be an $s$-by-$b$ matrix with Rademacher random entries,
and let $\mb U=\mb H \odot \mb \Sigma$.
Then, the probability that the left nullspace of $\mb U$ contains a fully dense
  vector is at most
\[
  2^{-b + s \log ( e^{2} b / s)}
\]
\end{lemma}
\begin{proof}
As in the preceding lemma, we let
$\mb U=[\mb u_1|\dots|\mb u_b]$ denote the columns of $\mb U$ and 
  for each $j\in [b]$, we let $N_j$ be the left nullspace of
  $[\mb u_1|\dots|\mb u_j]$.
We will show that it is very unlikely that $N_{b}$ contains
  a fully dense vector.

To this end, we show that if $N_{j-1}$ contains a fully dense vector,
  then with probability at least $1/2$ the dimension of $N_{j}$
  is less than the dimension of $N_{j-1}$.
To be concrete, assume that the first $j-1$ columns of $\mb \Sigma$ have been
  fixed and that $N_{j-1}$ contains a fully dense vector.
Let $\mb \alpha$ be any such vector.
If $\mb u_{j}$ contains only one non-zero entry, then
  $\mb \alpha^{T} \mb u_{j} \not = 0$ and so the dimension of
  $N_{j}$ is less than the dimension of $N_{j-1}$.
If $\mb u_{j}$ contains more than one non-zero entry, each of its
  non-zero entries are random Rademacher random variables.
So, Theorem~\ref{thm:erdos} implies that the probability over the
  choice of entries in the $j$th column of $\mb \Sigma$ that  
  $\mb \alpha^{T} \mb u_{j} = 0$ is at most one-half.
So, with probability at least $1/2$ the dimension of $N_{j}$
  is less than the dimension of $N_{j-1}$.

To finish the proof, we observe that the dimension of the nullspaces
  cannot decrease more than $s$ times.
In particular, for $N_{b}$ to contain a fully dense vector,
  there must be at least $b - s $ columns for which the dimension
  of the nullspace does not decrease.
Let $F \subset [b]$ have size $b - s $.
The probability that for each $j \in F$ that
  $N_{j-1}$ contains a fully dense vector and
  that the dimension of $N_{j}$ equals the dimension
  of $N_{j-1}$ is at most
  $2^{-b+s-1}$.
Taking a union bound over the choices for $F$,
  we see that the probability that $N_{b}$
  contains a fully dense vector is at most
\[
  \binom{b}{b -s} 2^{-b + s}
=
  \binom{b}{s} 2^{-b + s}
\leq 
  \left( \frac{e b}{s} \right)^{s}
  2^{-b + s}
\leq
  2^{-b + s + s \log (e b / s)}
\leq
  2^{-b + s \log (e^{2} b / s)}.
\]
\end{proof}

\begin{proof}[Proof of Lemma~\ref{lem:nnzRademacher}]
Notice that if $\mb X = \mb \Omega \odot \mb R$ is Bernoulli-Subgaussian, then because the entries of $\mb R$ are symmetric random variables, $\mb X$ is equal in distribution to 
$\mb \Omega \odot \mb R \odot \mb \Sigma$, where $\mb \Sigma$ is an independent iid Rademacher matrix. We will apply Lemma \ref{lem:uniq_2.5R} with $\mb H = \mb \Omega \odot \mb R$. 

If there is a fully-dense vector $\mb \alpha $
  for which $\norm{\mb \alpha^{T} \mb U}_{0} \leq t/5$,
  then there is a subset of at least $b = 4 t /5$
  columns of $\mb U$ for which
  $\mb \alpha$ is in the nullspace of the restriction of $\mb U$
  to those columns.
By Lemma~\ref{lem:uniq_2.5R}, the probability that this happens
  for any particular subset of $b$ columns is at most
\[
  2^{-b + s \log e^{2} b / s} \leq 
  2^{-4t/5 + s \log ( e^{2} t / s)}.
\]
Taking a union  bound over the subsets of $b$ columns,
  we see that the probability that this can happen is at most
\[
\binom{t}{4t/5}
  2^{-4t/5 + s \log e^{2} t / s}
\leq 
  2^{0.722 t}
  2^{-t (4/5 - (s/t) \log (e^{2} t / s))}
\leq 
  2^{t (0.722 - 0.8 + 0.0365)}
\leq 
  2^{-t/25},
\]
where in the first inequality we bound the binomial coefficient using
  the exponential of the corresponding binary entropy function, and in the
  second inequality we exploit $s/t < 1/200$.
\end{proof}

\subsection{Proof of Lemma \ref{lem:uniqueLastGR}}

\begin{proof}
Rather than considering vectors, we will consider the sets on
  which they are supported.
So, let $S \subseteq [n]$ and let $\sigma = \sizeof{S}$.
We first consider the case when $17 \leq \sigma  \leq 1/\theta $.
Let $T$ be the set of columns of $\mb X$ that have non-zero entries
  in the rows indexed by $S$.
Let $t = \sizeof{T}$.
By Lemma~\ref{lem:uniqueMore},
\[
  \prob{t < (3/8) \sigma \theta p} \leq \exp (- \sigma \theta p / 64).
\]
Given that $t \geq (3/8) \sigma \theta p$,
  Lemma~\ref{lem:nnzRademacher} tells us that the probability that
  there is a vector $\mb \alpha$ with support exactly $S$ for which
\[
  \norm{\mb \alpha^{T} \mb X}_{0} < (11/9) \theta p \leq  (3/40) \sigma \theta p
\]
is at most
\[
   \exp (- (3/200 )\sigma \theta p). 
\]  
Taking a union bound over all sets $S$ of size $\sigma$, we see that
  the probability that there vector $\mb \alpha$ of support size $\sigma$
  such that $ \norm{\mb \alpha^{T} \mb X}_{0} < (11/9) \theta p$
  is at most
\[
\binom{n}{\sigma }
  \left(
 \exp (- (3/200 )\sigma \theta p) + \exp (- \sigma \theta p / 64)
 \right)
\leq 
\exp (- c \sigma \theta p),
\]
for some constant $c$ given that $p > C n \log n$ for a sufficiently large $C$.

For $\sigma \geq 1/\theta$,
we may follow a similar argument to show that the probability that
  there is a vector $\mb \alpha$ with support size $\sigma$ for which
  $ \norm{\mb \alpha^{T} \mb X}_{0} < (11/9) \theta p$
  is at most
\[
  \exp (- c p),
\]
for some other constant $c$.
Summing these bounds over all $\sigma$ between $17$ and $n$,
  we see that the probability
  that there exists a vector $\mb \alpha$ with support of size at least
  $17$ such that
  such that $ \norm{\mb \alpha^{T} \mb X}_{0} < (11/9) \theta p$
  is at most
\[
\exp (- c \theta p),
\]
  for some constant $c$.

To finish, we sketch a proof of how we handle the sets of support between
 $2$ and $17$.
For $\sigma$ this small and for $\theta$ sufficiently small relative to $\sigma$
  (that is smaller than some constant depending on $\sigma$),
  each of the columns in $T$ probably has exactly one non-zero entry.
Again applying a Chernoff bound and a union bound over the choices of $S$,
  we can show that with probability $1 - \exp (- c \theta p)$
  for every vector $\mb \alpha$ with support of size between $2$ and $17$,
  $ \norm{\mb \alpha^{T} \mb X}_{0} \geq  (5/4) \theta p$.
\end{proof}

\subsection{Proof of Theorem \ref{thm:uniq}}

\begin{proof}
From Lemma~\ref{lem:uniqueLastGR} we know that with probability at most $\exp(- c \theta p)$, any dense linear combination of two or more rows of $\mb X$ has at least $(11/9) \theta p$ nonzeros. Hence, the $n$ rows of $\mb X$ are the sparsest directions in the row space of $\mb Y$. 

A Chernoff bound shows that the probability that
  any row of $\mb X$ has more than
\[
  (10/9) \theta p
\]
 non-zero entries is at most
\[
n \exp\left(-\frac{\theta p}{243}\right).
\]
Hence, with the stated probability, the rows of $\mb X$ are the $n$ sparsest vectors in $\mathrm{row}(\mb X)$. 

On the aforementioned event of probability at least $1-\exp(-c \theta p)$, $\mb X$ has no left null vectors with more than one nonzero entry. So, as long as all of the rows of $\mb X$ are nonzero, $\mb X$ will have no nonzero vectors in its left nullspace. With probability at least $1- n (1-\theta)^p \ge 1 - n \exp( -cp ) $, all of the rows of $\mb X$ are nonzero, and so $\mathrm{row}(\mb X) = \mathrm{row}(\mb Y) = \mathrm{row}(\mb X')$. 

This, together with our previous observations implies that every vector in $\mathrm{row}(\mb X')$ is a scalar multiple of a row of $\mb X$, from which uniqueness follows. Summing failure probabilities gives the quoted bound. 
\end{proof}

\section{Proof of Correct Recovery} \label{sec:recovery-proof}

Our analysis will proceed under the following probabilistic assumption on the coefficient matrix $\mb X$:

\paragraph{Notation.} Below, we will let $\| \mb M \|_{r1} \doteq \max_i \| \mb e_i^T \mb M \|_1$, where the $\mb e_i$ are the standard basis vectors. That is to say, $\| \cdot \|_{r1}$ is the maximum row $\ell^1$ norm. This is equal to the $\ell^1 \to \ell^1$ operator norm of $\mb M^T$. In particular, for all $\mb v$, $\mb M$, $\| \mb v^T \mb M \|_1 \le \| \mb v \|_1 \| \mb M \|_{r1}$. 

\subsection{Proof of Lemma \ref{lem:reduction-new}}\label{ssec:reduction}

\newcommand{\trunc}[1]{\mc T_{#1}}

\begin{proof} We will invoke a technical lemma (Lemma \ref{lem: ext-pos}) which applies to Bernoulli-Subgaussian matrices whose elements are bounded almost surely. For this, we define a truncation operator $\trunc{\tau} : \Re^{n \times p}\to\Re^{n\times p}$ via 
$$(\trunc{\tau} [ \mb M ])_{ij} = \left\{ \begin{array}{cc} M_{ij} & | M_{ij} | \le \tau \\ 0 & \text{else} \end{array} \right.$$
That is, $\trunc{\tau}$ simply sets to zero all elements that are larger than $\tau$ in magnitude. We will choose $\tau =\sqrt{24 \log p}$ and set
\begin{equation}
\mb X' \;=\; \trunc{\tau}[ \mb X ] \;=\; \mb \Omega \odot \trunc{\tau}[ \mb R ] \;\doteq\; \mb \Omega \odot \mb R'. 
\end{equation}
The elements of $\mb R'$ are iid symmetric random variables. They are bounded by $\tau$ almost surely, and have variance at most $1$. Moreover, 
\begin{eqnarray*}
\mu' &\doteq& \E\left[ \, | R'_{ij} | \, \right] \;=\; \int_{t= 0}^\infty \P[ | R'_{ij} | \ge t ] \,dt \quad=\quad \int_{t = 0}^\infty \P[ | R_{ij} | \ge t ] \, dt - \int_{t = \tau}^\infty \P[ |R_{ij} | \ge t ] \, dt \\
&\ge& \mu - 2 \int_{t = \tau}^\infty \exp\left( - \frac{t^2}{2} \right) dt \quad\ge\quad \mu - 2 p^{-12} \quad > \quad 1/20.
\end{eqnarray*}
The final bound follows from provided the constant $C$ in the statement of the lemma is sufficiently large. Of course, since $\mu \le 1$, we also have $\mu' \le 1$. 

The random matrix $\mb X'$ is equal to $\mb X$ with very high probability. Let
\begin{equation}
\event_X = \textbf{event}\{ \mb X = \mb X' \}.
\end{equation}
We have that 
\begin{equation}
\prob{\event_X^c} \;=\;\prob{ \exists \,(i,j) \mid |X_{ij}| > \tau } \;\le\; 2 np \, \exp\left( - \frac{\tau^2}{2} \right).
\end{equation}
Ensuring that $C$ is a large constant (say, $C \ge 2$ suffices), we have $p > n$. Since $\tau = \sqrt{24 \log p}$, $\prob{\event_X^c} \le 2p^{-10}$. 

For each $I \subseteq [n]$ of size at most $s = 1 / 8\theta$, we introduce two ``good'' events, $\event_{S}(I)$ and $\event_{N}(I)$. We will show that on $\event_X \cap \event_{S}(I) \cap \event_{N}(I)$, for any $\mb b$ supported on $I$, and any optimal solution $\mb z_\star$, $\mathrm{supp}(\mb z_\star) \subseteq I$. Hence, the desired property will hold for all sparse $\mb b$ on 
\begin{equation}
\event_{\text{good}} = \event_X \cap \bigcap_{|I| \le s} \event_S(I) \cap \event_N(I).
\end{equation}

For fixed $I$, write $T(I) = \{ j \mid X_{i,j} = 0 \; \forall \; i \in I \}$, and set $S(I) = [p] \setminus T(I)$. That is to say, $T$ is the set of indices of columns of $\mb X$ whose support is contained in $I^c$, and $S$ is its complement (indices of those columns that have a nonzero somewhere in $I$). The event $\event_S(I)$ will be the event that $S(I)$ is not too large:
\begin{equation}
\event_S(I) = \textbf{event}\{|S(I)| < p/4 \}. 
\end{equation}
The event $\event_N(I)$ will be one on which the following holds:
\begin{eqnarray*}\label{eqn:subset-lb}
\forall \mb v \in \Re^{n-|I|}, \quad \|\mb v^T \mb X'_{I^c,\ast}\|_1-2\|\mb v^T \mb X_{I^c,S(I)}' \|_1 \;>\; c_{1} \, p\,\mu' \sqrt{\frac{\theta}{n}} \; \|\mb v\|_1.
\end{eqnarray*}
Since $\mu' > 1/20$, this implies
\begin{eqnarray}\label{eqn:subset-lb-2}
\forall \mb v \in \Re^{n-|I|}, \quad \|\mb v^T \mb X'_{I^c,\ast}\|_1-2\|\mb v^T \mb X_{I^c,S(I)}' \|_1 \;>\; c_2 \, p \, \sqrt{\frac{\theta}{n}} \; \| \mb v \|_1.
\end{eqnarray}
Obviously, on $\event_X \cap \event_N(I)$, the same bound holds with $\mb X'$ replaced by $\mb X$. Lemma \ref{lem: ext-pos} shows that provided $S$ is not too large, $\event_N(I)$ is likely to occur: $\P[ \event_N(I) \mid \event_S(I) ]$ is large. 

On $\event_X \cap \event_S(I) \cap \event_N(I)$, we have inequality \eqref{eqn:subset-lb-2}. We show that this implies that if $\mb b$ is sparse, for any solution $\mb z_\star$, $\mathrm{supp}(\mb z_\star) \subseteq \mathrm{supp}(\mb b)$. Consider any $\mb b$ with $\mathrm{supp}(\mb b) \subseteq I$, and any putative solution $\mb z_\star$ to the optimization problem \eqref{eqn:main-opt-X}. If $\mathrm{supp}(\mb z_\star) \subseteq I$, we are done. If not, let $\mb z_0 \in \Re^n$ such that 
\begin{equation}
[z_0]_i = \left\{ \begin{array}{cc} [z_\star]_i & i \in I \\ 0 & \text{else} \end{array} \right.,
\end{equation}
and set $\mb z_1 = \mb z_\star - \mb z_0$. Notice that since $\mb b^T \mb z_1 = 0$, $\mb z_0$ is also feasible for \eqref{eqn:main-opt-X}. We prove that under the stated hypotheses $\mb z_1 = \mb 0$. 

Form a matrix $\mb X_S' \in \Re^{n \times p}$ via
\begin{equation}
[X_S']_{ij} = \left\{ \begin{array}{cc} X_{ij}' & j \in S \\ 0 & \text{else} \end{array} \right.,
\end{equation}
and set $\mb X_T' = \mb X' - \mb X_S'$. We use the following two facts: First, since $S$ and $T$ are disjoint, for any vector $\mb q$, $\| \mb q \|_1 = \| \mb q_S \|_1 + \| \mb q_T\|_1$. Second, by construction of $T$, $\mb z_0^T \mb X_T = \mb 0$. Hence, we can bound the objective function at $\mb z_\star = \mb z_0 + \mb z_1$ below, as 
\begin{align}
\|(\mb z_0+\mb z_1)^T \mb X\|_1&=\|(\mb z_0+\mb z_1)^T\mb X_S \|_1 + \| (\mb z_0 + \mb z_1)^T \mb X_T\|_1\nonumber\\
&\ge \| \mb z_0^T \mb X_S \|_1  - \| \mb z_1^T \mb X_S \|_1 + \| \mb z_1^T \mb X_T \|_1.
\end{align}
Since $\mb z_0^T \mb X = \mb z_0^T (\mb X_S + \mb X_T) = \mb z_0^T \mb X_S$, this bound is equivalent to 
\begin{equation} \label{eqn:norm-final}
\|(\mb z_0+\mb z_1)^T \mb X\|_1 \ge \| \mb z_0^T \mb X \|_1 + \| \mb z_1^T \mb X \|_1 - 2 \| \mb z_1^T \mb X_S \|_1
\end{equation}
Noting that $\mb z_1$ is supported on $I^c$, \eqref{eqn:subset-lb} implies that if $\mb z_1 \ne \mb 0$, $\|\mb z_0^T \mb X \|_1$ is strictly smaller than $\|\mb z_\star^T \mb X \|_1 = \| (\mb z_0+\mb z_1)^T \mb X \|_1$. Hence, $\mb z_0$ is a feasible solution with objective strictly smaller than that of $\mb z_\star$, contradicting optimality of $\mb z_\star$. To complete the proof, we will show that $\P[ \event_{\text{good}} ]$ is large. 

\paragraph{Probability.} 

The subset $S$ is a random variable, which depends only on the rows of $\mb X_{I,\ast}$ of $\mb X$ indexed by $I$. For any fixed $I$, $|S| \sim \mathrm{Binomial}(p,\lambda)$, with $\lambda \doteq 1-(1-\theta)^{|I|}$. We have $\theta \le \lambda \le |I|\theta$, where the upper bound uses convexity of $(1-\theta)^{|I|}$. From our assumption on $s$, $\lambda \le |I| \theta \le 1/8$. Applying a Chernoff bound, we have  
\begin{equation} \label{eqn:S-bound-lower}
\P\left[ \, |S| \ge p/4 \,\right] \;\le\; \P\left[ \, |S|\geq 2\lambda p \, \right] \;\leq\;  \exp\left( - \frac{\lambda p}{3} \right) \;\leq\; \exp\left(- \frac{\theta p}{3} \right).  
\end{equation}
Hence, with probability at least $1-\exp\left(-\frac{\theta p}{3} \right)$, we have $|S| < p/4$. 

Since $\mb X'$ is iid and $S$ depends only on $\mb X'_{I,\ast}$, conditioned on $S$, $\mb X_{I^c,\ast}'$ is still iid Bernoulli-Subgaussian. Applying Lemma \ref{lem: ext-pos} to $\mb X_{I^c,\ast}'$, conditioned on $S$ gives
\begin{eqnarray} \label{eqn:eNc-bound}
\prob{\event_N(I)^c \mid \event_S(I)} \;\le\; \exp\left( -\frac{c p}{n\sqrt{\log p}} + n \log \left(C n \sqrt{\log p} \right) \right). 
\end{eqnarray}
In this bound, we have used that $\mu' \in [1/20,1]$, $\tau = \sqrt{24 \log p}$, and the fact that the bound in Lemma \ref{lem: ext-pos} is monotonically increasing in $n$ to simplify the failure probability. Moreover, we have 
\begin{eqnarray*}
\prob{\event_S(I) \cap \event_{N}(I)} &=& 1 - \P[ \event_N(I)^c \mid \event_S(I) ] \P[\event_S(I) ]  \;-\; \P[ \event_S(I)^c ]\\
 &\ge& 1 - \P[\event_N(I)^c \mid \event_S(I) ] \;-\; \P[\event_S(I)^c] \\
 &\ge& 1 - \exp\left( -\frac{c p}{n\sqrt{\log p}} + n \log \left(C n \sqrt{\log p} \right) \right) - \exp\left(- \frac{\theta p}{3} \right).
\end{eqnarray*}
Let $s_{\max} = \left\lfloor\frac{1}{8\theta} \right\rfloor$ denote the largest value of $|I|$ allowed by the conditions of the lemma. 
\begin{eqnarray}
\P[\event_{\text{good}}^c] &\le& \P[\event_X^c] + \sum_{I \subset [n],\; |I| \le s_{max}} \P[ (\event_S(I) \cap \event_N(I))^c ]  \nonumber  \\ &\le& \P[\event_X^c] + s_{\max} \binom{n}{s_{\max}} \left\{ \exp\left(- \frac{\theta p}{3} \right) + \exp\left( -\frac{c p}{n\sqrt{\log p}} + n \log \left(C n \sqrt{\log p} \right) \right) \right\} \nonumber \\
&\le& 2p^{-10} + s_{\max} \binom{n}{s_{\max}} \left\{ \exp\left(- \frac{2 p}{3n} \right) +  \exp\left( -\frac{c p}{n\sqrt{\log p}} + n \log \left(C n \sqrt{\log p} \right) \right) \right\}\nonumber\\
&\le& 2p^{-10} +  \exp\left( -\frac{cp}{n\sqrt{\log p}} + C'' n \log \left(n\log p\right) \nonumber \right),
\end{eqnarray}
for appropriate constant $C''$. Under the conditions of the lemma, provided $C$ is large enough, the final term can be bounded by $p^{-10}$, giving the result. 
\end{proof}

\begin{lemma} \label{lem:avg_lower_bound} Suppose $\mb x = \mb \Omega \odot \mb R \in \Re^n$ with $\mb \Omega$  iid $\mathrm{Bernoulli}(\theta)$, $\mb R$ an independent random vector with iid symmetric entries, $n\theta\geq 2$, and $\mu \doteq \E[ \; |R_{ij}| \; ] < + \infty$. Then for all $\mb v \in \Re^n$,
\begin{align}
\E\left[ \left|\mb v^T \mb x \right| \right] \;\geq\;   \frac{ \mu }{4} \sqrt{\frac{\theta}{n}}\|\mb v\|_1.
\end{align}
\end{lemma}
\begin{proof}
Let $\mb z= \mb v/\| \mb v\|_1$, and write
\begin{eqnarray*}
\xi &=& \E \Bigl[ \bigl| \sum_{j=1}^n x_{j} z_j \bigr| \Bigr] = \E \Bigl[ \bigl| \sum_{j=1}^n \Omega_j R_j z_j \bigr| \Bigr], \\ 
&\ge& \inf_{\| \mb \zeta \|_1 = 1} \E \Bigl[ \bigl| \sum_{j=1}^n \Omega_j R_j \zeta_j \bigr| \Bigr]
\;=\; \inf_{\mb \zeta \ge 0, \sum_j \zeta_j = 1} \E\Bigl[ \bigl| \sum_j \Omega_j R_j \zeta_j \bigr| \Bigr],
\end{eqnarray*}
where the final equality holds due to the fact that $\mathrm{sign}(\zeta_j) R_j$ is equal to $R_j$ in distribution. Notice that $\E\left[ | \sum_j \Omega_j R_j \zeta_j | \right]$ is a convex function of $( \zeta_j )$, and is invariant to permutations. Hence, this function is minimized at the point $\zeta_1 = \zeta_2 = \dots = \zeta_n = 1/n$, and 
\begin{equation}
\xi \;\ge\; n^{-1} \E\Bigl[ \bigl| \sum_j \Omega_j R_j \bigr| \Bigr] \;=\; n^{-1} \E_\Omega \E_R \Bigl[ \bigl| \sum_j \Omega_j R_j \bigr| \Bigr].
\end{equation}
Let $T = \#\{ j \mid \Omega_j = 1 \}$. For fixed $\Omega$, $\sum_j \Omega_j R_j$ is a sum of symmetric random variables. Thus
\begin{equation}
\E_R \Bigl[ \bigl| \sum_j \Omega_j R_j \bigr| \Bigr] \;=\; \E_R \E_\varepsilon  \Bigl[ \bigl| \sum_j \Omega_j |R_j| \varepsilon_j \bigr| \Bigr],
\end{equation}
where $(\varepsilon)$ is an independent sequence of Rademacher (iid $\pm 1$) random variables. By the Khintchine inequality, 
\begin{equation}
\E_R \E_\varepsilon  \Bigl[ \bigl| \sum_j \Omega_j |R_j| \varepsilon_j \bigr| \Bigr] \ge \frac{1}{\sqrt{2}} \E_R \| \mb \Omega \odot \mb R \|_2 \ge \frac{1}{\sqrt{2T}} \E_R \| \mb \Omega \odot \mb R \|_1  = \mu \sqrt{T/2}. 
\end{equation}
Hence
\begin{equation} \label{eqn:mu-sl}
\xi \;\ge\; \frac{\mu}{n} \sum_{s = 0}^n \P[ T = s ] \sqrt{\frac{s}{2}}.
\end{equation}
Notice that $T \sim \mathrm{Binomial}(n,\theta)$, and hence, with probability at least $1/2$, $T \ge \lfloor n \theta \rfloor \ge n \theta / 2$, where the last inequality holds due to the assumption $n \theta \ge 2$. 
Plugging in to \eqref{eqn:mu-sl}, we obtain that
\begin{equation}
\xi \;\ge\; \frac{\mu}{4} \sqrt{\frac{\theta}{n}}.
\end{equation}
\end{proof}

\begin{lemma} \label{lem: ext-pos} Suppose that $\mb X$ follows the Bernoulli-Subgaussian model, and further that for each $(i,j)$, $| X_{ij} | \le \tau$ almost surely. Let $\mb X_S$ be a column submatrix indexed by a fixed set $S\subset [p]$ of size $|S| < \frac{p}{4}$. There exist positive numerical constants $c$, $C$ such that the following:
\begin{align}
\|\mb v^T \mb X\|_1-2\|\mb v^T\mb X_S\|_1 \;>\; \frac{ \mu p}{32} \sqrt{\frac{\theta}{n}}\,\|\mb v\|_1\label{eqn:Mz-sum}
\end{align}
holds simultaneously for all $\mb v \in \Re^n$, on an event  $\event_N$ with 
$$\prob{\event_N^c} \;\le\;  \exp\left( - \frac{c \mu^2 p}{n (1 + \tau \mu)} + n \log\left( C \max\left\{ \frac{n \tau}{\mu}, 1 \right\} \right) \right).$$
\end{lemma}

\begin{proof} Consider a fixed vector $\mb z$ of unit $\ell^1$ norm. Note that 
\begin{equation}
\| \mb z^T \mb X \|_1 - 2 \| \mb z^T \mb X_S \|_1 \;=\; \sum_{i \in S^c} |\mb z^T \mb X_i| - \sum_{i\in S} | \mb z^T \mb X_i| 
\end{equation}
is a sum of $p$ independent random variables. Each has absolute value bounded by 
$$|\mb z^T \mb X_i| \;\leq\; \|\mb z\|_1\|\mb X\|_\infty \;\leq\; \tau,$$ 
and second moment 
$$\text{Var}(|\mb z^T \mb X_i|)\;\leq\; \E[(\mb z^T \mb X_i)^2] \;\leq\; \theta \|\mb z\|_2^2 \;\leq\; \theta.$$ 
There are $p$ such random variables, so the overall variance is bounded by $p \theta$. The expectation of is 
\begin{eqnarray*}
\E\left[ \| \mb z^T \mb X \|_1 - 2 \| \mb z^T \mb X_S \|_1\right] = (p - 2 |S|) \E | \mb z^T \mb X_1 | = (p - 2|S|) \tilde{\mu} \ge \frac{\tilde{\mu} p}{2},
\end{eqnarray*}
where $\tilde{\mu} = \E | \mb z^T \mb X_1 |$. We apply Bernstein's inequality to bound the deviation below the expectation: 
\begin{equation} \label{eqn:Hoeffding}
\P\left[ \,\| \mb z^T \mb X \|_1-2\|\mb z^T \mb X_S\|_1 \;\le\; \frac{\tilde{\mu}p}{2} - t \, \right] \;\le\; \exp\left(-\frac{t^2}{2 \theta p + 2 \tau t /3}  \right).
\end{equation}
We will set $t \;=\; \frac{\mu}{16} \sqrt{ \frac{\theta}{n} } p$, and notice that by Lemma \ref{lem:avg_lower_bound}, $\tilde{\mu} \ge \frac{\mu}{4} \sqrt{\frac{\theta}{n}}$. Hence, we obtain 
\begin{equation}
\P\left[ \,\| \mb z^T \mb X \|_1-2\|\mb z^T \mb X_S\|_1 \;\le\; \frac{\mu}{16} \sqrt{\frac{\theta}{n}} \, p \; \right]  \;\le\; \exp\left(-\frac{c_1 \mu^2 \theta p^2 /n}{2 \theta p + c_2 p \mu \tau \sqrt{\theta/n} }  \right).
\end{equation}
Using that $\theta/n \le \theta^2/2$ (from the assumption $\theta n \ge 2$), $\mu \le 1$, and simplifying, we can obtain a more appealing form:
\begin{eqnarray}
\P\left[ \,\| \mb z^T \mb X \|_1-2\|\mb z^T \mb X_S\|_1 \le \frac{\mu}{16} \sqrt{\frac{\theta}{n}} \, p \, \right] \;\le\; \exp\left(-\frac{c\mu^2 p}{n (1 + \tau \mu) }  \right),
\end{eqnarray}
This bound holds for any fixed vector of unit $\ell^1$ norm. We need a bound that holds simultaneously for {\em all} such $\mb z$. For this, we employ a discretization argument. This argument requires two probabilistic ingredients: bounds for each $\mb z$ in a particular net, and a bound on the overall $\ell^1$ operator norm of $\mb X^T$, which allows us to move from a discrete set of $\mb z$ to the set of all $\mb z$. We provide the second ingredient first: 

Now, let $N$ be an $\eps$-net for the unit ``1-sphere'' $\Gamma \doteq \{ \mb x \mid \| \mb x \|_1 = 1 \}$.
That is, for all $\norm{\mb x}_{1} = 1$, there is a $\mb z \in N$ such that $\norm{\mb x- \mb z}_{1} \leq \epsilon$.
For example, we could take
\[
  N = \setof{\mb y / \ceil{n/\epsilon } : \mb y \in \mathbb Z^{n} , \norm{\mb y}_{1} = \ceil{n/\epsilon} }.
\]
With foresight, we choose $\eps = \frac{\mu}{32 \tau} \sqrt{\frac{\theta}{ n}} \;\ge\; \frac{\mu \sqrt{2}}{32 \tau n}$. The inequality follows from $\theta n \ge 2$. Using standard arguments (see, e.g. \cite[Page 15]{Stanley}), one can show
\[
  \sizeof{N} \leq 2^{n} \binom{\ceil{n/ \epsilon} + n - 1}{n}
\leq 
  (2 e (\ceil{1/\epsilon} + 1))^{n}.
\]
So,
\begin{equation} \label{eqn:msv-N}
\log |N| \;\leq\;  n \log\left( C \max\left\{ \frac{n \tau}{\mu}, 1\right\} \right),
\end{equation}
for appropriate constant $C$.
 
 For each $\mb z \in N$, set
 \begin{equation}
 \event_{\mb z} = \textbf{event}\left\{ \| \mb z^T \mb X \|_1 - 2 \| \mb z^T \mb X_S \|_1 > \frac{\mu}{16} \sqrt{\frac{\theta}{n}} \, p \right\}.
 \end{equation}
 Our previous efforts show that for each $\mb z$, 
 \begin{equation}
 \prob{ \event_{\mb z}^c } \;\le\; \exp\left(-\frac{c\mu^2 p}{n (1 + \tau \mu) }  \right). 
 \end{equation}
  Set 
 \begin{equation}
  \event_N =  \bigcap_{\mb z \in N} \event_{\mb z}.
 \end{equation}
 On $\event_N$, consider any $\mb v \in \Gamma$, choose $\mb z \in N$ with $\| \mb v - \mb z \|_1 \le \epsilon$, and set $\mb \Delta = \mb v - \mb z$, then 
\begin{align}\label{eqn:net-expansion}
\|\mb v^T \mb X\|_1-2\|\mb v^T\mb X_S\|_1 \; &= \; \|(\mb z+\mb \Delta)^T\mb X_{S^c}\|_1-\|(\mb z+\mb \Delta)^T \mb X_S\|_1 \nonumber \\
&\ge \; \| \mb z^T \mb X_{S^c} \|_1 - \| \mb z^T \mb X_S \|_1 - \| \mb \Delta^T \mb X_{S^c}\|_1 -  \| \mb \Delta^T\mb  X_S\|_1 \nonumber \\
&= \; \| \mb z^T \mb X \|_1 - 2 \| \mb z^T \mb X_S \|_1 - \| \mb \Delta^T \mb X' \|_1 \nonumber \\
&\ge \; \| \mb z^T \mb X \|_1 - 2 \| \mb z^T \mb X_S \|_1 - \| \mb \Delta \|_1 \| {\mb X}^T \|_1 \nonumber \\ 
&\ge \; \frac{\mu}{16} p \sqrt{\frac{\theta}{n}}- \eps p \tau \nonumber \\
&= \; \left(\frac{\mu}{16} - \eps \tau \sqrt{\frac{n}{\theta}} \right) \sqrt{\frac{\theta}{n}} \, p \nonumber \\
&= \; \frac{\mu}{32} \sqrt{\frac{\theta}{n}} \, p.
\end{align} 
This bound holds all $\mb v \in \Gamma$ (above, we have used that for $\mb v \in \Gamma$, $\| \mb v \|_1 = 1$). By homogeneity, whenever this inequality holds over $\Gamma$, it holds over all of $\Re^n$ we have 
\begin{equation}
\|\mb v^T \mb X\|_1-2\|\mb v^T\mb X_S\|_1 \;\ge\; \frac{\mu}{32} \sqrt{\frac{\theta}{n}} \, p \, \| \mb v \|_1. 
\end{equation}
To complete the proof, note that the failure probability is bounded as 
\begin{eqnarray*}
\prob{ \event_N^c } &\le& \sum_{\mb z \in N} \prob{ \event_{\mb z}^c } \quad\le\quad |N| \times \exp\left( - \frac{c \mu^2 p}{n (1 +  \tau \mu)} \right) \\
&\le& \exp\left( - \frac{c \mu^2 p}{n (1 +  \tau \mu)} + n \log\left( C \max \left\{ \frac{ n \tau}{\mu}, 1\right\} \right) \right).
\end{eqnarray*}
\end{proof}

\subsection{Proof of Lemma \ref{lem:extreme-sparse-new}}

\vspace{.1in}

\begin{proof} For each $j \in [n]$, set $T_j = \{ i \mid X_{ji} = 0 \}$. Set $$\Omega_{J,j} = \left\{ \ell \mid X_{j,\ell} \ne 0, \;\text{and} \; X_{j',\ell} = 0, \;\forall j' \in J \setminus\{j\} \right\}.$$ 
Consider the following conditions: 
\begin{eqnarray}
&& \| \mb X \|_{r1} \le (1+\eps) \mu \theta p,  \label{eqn:ub-cond-1} \\
\forall \,j\in [n], &&\|\mb X_{[n] \setminus \{j\}, T_j^c} \|_{r1} \;\le\; \alpha \mu \theta p,  \label{eqn:ub-cond-2} \\
\forall \,J \in \binom{[n]}{s}, j \in J, && \| \mb X_{j,\Omega_{J,j}} \|_1 \;\ge\; \beta \mu \theta p. \label{eqn:ub-cond-3}
\end{eqnarray}
We show that when these three conditions hold for appropriate $\eps,\alpha,\beta > 0$, the property {\bf (P2)} will be satisfied, and any solution to a restricted subproblem will be $1$-sparse. Indeed, fix $J$ of size $s$ and nonzero $\mb b \in \Re^n$, with $\mathrm{supp}(\mb b) \subseteq J$. Let $j^\star \in J$ denote the index of the largest element of $\mb b$ in magnitude. 

Consider any $\mb v$ whose support is contained in $J$, and which satisfies $\mb b^T \mb v = 1$. Then we can write $\mb v = \mb v_0 + \mb v_1$, with $\mathrm{supp}(\mb v_0) \subseteq \{ j^\star \}$ and $\mathrm{supp}(\mb v_1) \subseteq J \setminus \{j^\star\}$. We have 
\begin{eqnarray}
\| \mb v^T \mb X_J \|_1 &=& \|\mb v_0^T \mb X_J + \mb v_1^T \mb X_J  \|_1 \nonumber \\
 &\ge& \| \mb v_0^T \mb X_J \|_1 + \| \mb v_1^T \mb X_{J, T_{j^\star}} \|_1 - \|\mb v_1^T \mb X_{J, T_{j^\star}^c} \|_1. \label{eqn:triangle-after}
\end{eqnarray}
Above, we have used that $\mb v_0$ is supported only on $j^\star$, $\mathrm{supp}(\mb v_0^T \mb X) = T_{j^\star}^c$, and applied the triangle inequality. 

\newcommand{\vt}{\tilde{\mb v}}

First without loss of generality, we assume by normalization $b_{j^\star}=1$. We know that $\mb v_0=(1-\mb b^T \mb v_1) \mb e_{j^\star}$, and the vector $\mb e_{j^\star}$ is well-defined and feasible. We will show that in fact, this vector is optimal. Indeed, from \eqref{eqn:triangle-after}, we have 
\begin{eqnarray*}
\| \mb v^T \mb X_J \|_1 &\geq& \| \mb (1- \mb b^T \mb v_1) \mb e_{j^\star}^T \mb X_J \|_1  + \| \mb v_1^T \mb X_{J,T_{j^\star}} \|_1 - \| \mb v_1^T \mb X_{J,T_{j^\star}^c} \|_1 \nonumber \\
&\ge& \| \mb e_{j^\star}^T \mb X_J \|_1 -  \|\mb X_J \|_{r1} \| \mb b_{J \setminus \{ j^\star \}} \|_\infty \| \mb v_1 \|_1 + \| \mb v_1^T \mb X_{J,T_{j^\star}} \|_1 - \| \mb v_1 \|_1 \| \mb X_{J,T_{j^\star}^c} \|_{r1} \nonumber \\
&\ge& \| \mb e_{j^\star}^T \mb X_J \|_1 + \| \mb v_1^T \mb X_{J,T_{j^\star}} \|_1 - \left( \| \mb X_J \|_{r1} \times \left(1-\gamma \right) + \| \mb X_{J,T_{j^\star}^c} \|_{r1} \right) \| \mb v_1 \|_1 
\end{eqnarray*}
In the last simplification, we have used that $|b|_{(2)}/|b|_{(1)} \leq 1-\gamma$. We can lower bound $\| \mb v_1^T \mb X_{J,T_{j^\star}} \|_1$ as follows: for each $i \in J$, consider those columns indexed by $j \in \Omega_{J,i}$. For such $j$, 
$$\mathrm{supp}(\mb v_1) \cap \mathrm{supp}(\mb X_{J,j}) = \{ i \},$$ 
and so 
$$\| \mb v_1^T \mb X_{J,\Omega_{J,i}} \|_1 = |v_1(i)| \| \mb X_{i,\Omega_{J,i}} \|_1.$$
Notice that  $i \ne i'$, $\Omega_{J,i} \cap \Omega_{J,i'} = \emptyset$, and that for $j \in J$, and $i \ne j$, $\Omega_{J,i} \subseteq T_j$. So, finally, we obtain 
\begin{equation}
\| \mb v_1^T \mb X_{T_{j^\star}} \|_1 \;\ge\; \sum_{i\in J} |v_1(i)| \| \mb X_{i,\Omega_{J,i}} \|_1.
\end{equation}
Plugging in the above bounds, we have that 
\begin{equation}
\| \mb v^T \mb X\|_1 \;\ge\; \|  \mb e_{j^\star}^T \mb X \|_1 + \left(\beta + \gamma - \alpha - \eps - 1 \right) \| \mb v_1 \|_1 \mu \theta p.
\end{equation}
Hence, provided $\beta - \alpha - \eps > 1 - \gamma$, $\mb e_{j^\star}$ achieves a strictly smaller objective function than $\mb v$. 

\paragraph{Probability.}

We will show that the desired events hold with high probability, with $\beta = 1 - \gamma/4$, and $\alpha = \eps = \gamma / 4$. Using Lemma \ref{lem:M-norm-bound-bg}, \eqref{eqn:ub-cond-1} holds with probability at least 
$$1 - 4 \exp\left( -\frac{c \gamma^2 \theta p}{16} + \log n \right).$$
For the second condition \eqref{eqn:ub-cond-2}, fix any $j$. Notice that for any $(j',i)$ the events $\mb X_{j',i} \ne 0$ and $i \in T_j^c$ are independent, and have probability $\theta^2$. Moreover, for fixed $j'$, the collection of all such events (for varying column $i$) is mutually independent. Therefore, for each $j \in [n]$, $\| \mb X_{[n] \setminus\{ j\}, T_j^c } \|_{r1}$ is equal in distribution to the $r1$ norm of an $(n-1) \times p$ matrix whose rows are iid $\theta^2$-Bernoulli-Subgaussian.\footnote{Distinct rows are not independent, since they depend on common events $i \in T_j^c$, but this will cause no problem in the argument.} Hence, by Lemma \ref{lem:M-norm-bound-bg}, we have 
\begin{align}
\P\left[ \| \mb X_{[n] \setminus\{j\},T_j^c} \|_{r1} > (1+\delta) \mu \theta^2 p \right] \le 4 \exp\left( - c \delta^2 \theta^2 p + \log (n -1) \right).\label{eqn:extreme-term2}
\end{align}
To realize our choice of $\alpha$, we need $\theta ( 1 + \delta ) = \gamma / 4$; we therefore set $\delta = \theta^{-1} \left( \gamma/ 4 - \theta \right) \ge \theta^{-1} \gamma / 8$. Plugging in, and taking a union bound over $j$ shows that \eqref{eqn:ub-cond-2} holds with probability at least 
\begin{equation}
1 - 4 \exp\left(- \frac{c \gamma^2 \theta p}{64} + 2\log n \right). 
\end{equation}
Finally, consider \eqref{eqn:ub-cond-3}. Fix $j$ and $J$. Notice for $i \in [p]$, the events $i \in \Omega_{J,j}$ are independent, and occur with probability $\theta'  = \theta (1-\theta)^{s-1} \ge \theta ( 1 - \theta s )$. Hence $\| \mb X_{j,\Omega_{J,j}} \|_1$ is distributed as the $\ell^1$ norm of a $1 \times p$ iid $\theta'$-Bernoulli-Gaussian vector. Again using Lemma \ref{lem:M-norm-bound-bg}, we have 
\begin{equation}
\P\left[  \| \mb X_{j,\Omega_{J,j}} \|_1 \le (1-\delta) \mu \theta' p \right] \;\le\; 4 \exp\left( - c \delta^2 \theta' p \right). \label{eqn:extreme-term3}
\end{equation}
To achieve our desired bound, we require $(1-\delta) \theta' \ge \beta \theta$; a sufficient condition for this is $(1-\delta)(1-\theta s) \ge 1 - \gamma/4$, or equivalently, 
$$\delta (1 - \theta s) \le \gamma / 4 - \theta s.$$
Using the bound $\theta s \le \gamma / 8$, we find that it suffices to set $\delta = \gamma / 8$. We therefore obtain 
\begin{equation}
\P\left[ \| \mb X_{j,\Omega_{J,j}} \|_1 \le \beta \mu \theta p \right] \;\le\; 4 \exp\left( -\frac{c \gamma^2 \theta' p}{64} \right).
\end{equation}
Finally, using $\theta' \ge \theta ( 1 - \theta s ) \ge \theta / 2$, we obtain
\begin{equation}
\P\left[ \| \mb X_{j,\Omega_{J,j}} \|_1 \le \beta \mu \theta p \right] \;\le\; 4 \exp\left( -\frac{c \gamma^2 \theta p}{128} \right).
\end{equation}
Taking a union bound over the $s \binom{n}{s}$ pairs $j,J$ and summing failure probabilities, we see that the desired property holds on the complement of an event of probability at most
\begin{eqnarray*}\label{eqn:1sparse-all-probs}
\lefteqn{ 4 s \binom{n}{s} \exp\left( -\frac{c \gamma^2 \theta p}{128} \right) + 4 \exp\left(- \frac{c \gamma^2 \theta p}{64} + 2 \log n \right) +  4 \exp\left( -\frac{c \gamma^2 \theta p}{16} + \log n \right) } \\
 &\le& 4 \left( s \binom{n}{s} + n^2 + n \right) \times \exp\left( -\frac{c \gamma^2 \theta p}{128} \right). \hspace{4in}
\end{eqnarray*}
Using that $\log \binom{n}{s} < s \log n$ and $\theta \ge 2 / n$, it is not difficult to show that under our hypotheses, the failure probability is bounded by $4 p^{-10}$.
\end{proof}

\begin{lemma} \label{lem:M-norm-bound-bg} Let $\mb X$ be an $n \times p$ random matrix, such that the marginal distributions of the rows $\mb X_i$ follow the Bernoulli-Subgaussian model. Then for any $ 0<\delta<1$ we have 
\begin{equation} \label{eqn:M-norm-bound}
(1-\delta) \mu \, \theta p\leq \| \mb X \|_{r1}\leq (1+\delta) \mu \, \theta p
\end{equation}
with probability at least
$$1-4\exp\left(-c \delta^2 \theta p + \log n\right).$$
where $c$ is a positive numerical constant. 
\end{lemma}

\begin{proof}
Let $\mb X_i = \mb \Omega_i \odot \mb R_i$ denote the $i$-th row of $\mb X$. We have $\| \mb X \|_{r1} = \max_{i \in [n]} \| \mb X_i \|_1$. For any fixed $i$, $\| \mb X_i \|_1 =\sum_j |X_{ij}|=\sum_j |\Omega_{ij}R_{ij}|$, and $\E[\|\mb X_i \|_1]=\sum_j \E|X_{ij}|= \mu \theta p$, where $\mu = \E[ |R_{ij}| ]$ is as specified in the definition of the Bernoulli-Subgaussian model. 

Let $T_i=\#\{j|\Omega_{ij}=1\}$. Since $T_i\sim \mathrm{Binomial}(p,\theta)$, for $0 < \delta_1 < 1$, the Chernoff bound yields
\begin{align}
\P\left[ \, T_i\geq (1+\delta_1)\theta p \, \right]\leq \exp\left(-\frac{\delta_1^2\theta p}{3}\right), \label{eqn:chernoff-upper} \\
\P\left[ \, T_i\leq (1-\delta_1)\theta p\, \right]\leq \exp\left(-\frac{\delta_1^2\theta p}{2} \right). \label{eqn:chernoff_lb_gn}
\end{align}
Since the $R_{ij}$ are subgaussian, so is $|R_{ij}|$. Conditioned on $\Omega_i$, $\| \mb X_i \|_1 = \sum_{j \in T_i} |R_{ij}|$ is a sum of independent subgaussian random variables. By \eqref{eqn:tail}, 
\begin{align}
\condprob{\magnitude{ \, \| \mb X_i \|_1- \mu T_i \, } \geq \delta_2 \mu \theta p}{ \Omega_i } \;\leq\; 2 \exp\left(-\frac{\delta_2^2 \mu^2 \theta^2 p^2}{2T_i}\right). 
\end{align}
Whenever $T_i \in \theta p \times [1-\delta_1,1+\delta_1]$, the conditional probability is bounded above by $\exp\left( - \frac{\delta_2^2 \mu \theta p}{2 (1-\delta_1)} \right)$. So, unconditionally, 
\begin{eqnarray*}
\lefteqn{\P\Bigl[ \, \| \mb X_i \|_1 \notin \mu \theta p \times [ 1- \delta_1 - \delta_2, 1 + \delta_1 +  \delta_2] \, \Bigr] } \\ 
&\le& 2 \exp\left( -\frac{c_1 \delta_2^2 \mu \theta p}{ (1-\delta_1)} \right) +  \exp\left( -\frac{\delta_1^2 \theta p}{3} \right) + \exp\left( - \frac{\delta_1^2 \theta p}{2} \right).
\end{eqnarray*}
Set $\delta_1 = \delta_2 = \delta / 2$, the fact that $\mu$ is bounded below by a constant, combine exponential terms, and take a union bound over the $n$ rows to complete the proof. 
\end{proof}

\subsection{Proof of Correct Recovery Theorem \ref{thm:correct}} \label{sec:recovery-proof-sc}

\begin{proof} If $n = 1$, the result is immediate. Suppose $n \ge 2$. We will invoke Lemmas \ref{lem:reduction-new} and \ref{lem:extreme-sparse-new}. The conditions of Lemma \ref{lem:reduction-new} are satisfied immediately. In Lemma \ref{lem:extreme-sparse-new}, choose $\gamma = \frac{\beta}{\log n}$, with $\beta$ smaller than the numerical constant $\alpha_0$ in the statement of Lemma \ref{lem:2nd-our-form}. Take $s = \lceil 6 \theta n \rceil$. Since $\theta n \ge 2$, $s \le 7 \theta n$. The condition $\theta s < \gamma / 8$ is satisfied as long as $$\theta^2 < \frac{\gamma}{56 n} = \frac{\beta}{56 \, n \log n}.$$
This is satisfied provided the numerical constant $\alpha$ in the statement of the theorem is sufficiently small. Finally, it is easy to check that $p$ satisfies the requirements of Lemma \ref{lem:extreme-sparse-new}, as long as the constant $c_1$ in the statement of the theorem is sufficiently large. 

So, with probability at least $1-7p^{-10}$, the matrix $\mb X$ satisfies properties {\bf (P1)}-{\bf (P2)} defined in Lemmas \ref{lem:reduction-new} and \ref{lem:extreme-sparse-new}, respectively. Consider the optimization problem 
\begin{equation} \label{eqn:main-t-pf-opt}
\text{minimize} \quad \| \mb w^T \mb Y \|_1 \quad \text{subject to} \quad \mb r^T \mb w = 1,
\end{equation}
with $\mb r = \mb Y \mb e_j$ the $j$-th column of $\mb Y$. This problem recovers the $i$-th row of $\mb X$ if the solution $\mb w_\star$ is unique, and $\mb A^T \mb w_\star$ is supported only on entry $i$. This occurs if and only if the solution $\mb z_\star$ to the modified problem 
\begin{equation}
\text{minimize} \quad \| \mb z^T \mb X \|_1 \quad \text{subject to} \quad \mb b^T \mb z = 1
\end{equation}
with $\mb b = \mb A^{-1} \mb r = \mb X \mb e_j$ is unique and supported only on entry $i$. 

Provided the matrix $\mb X$ satsifies properties {\bf (P1)}-{\bf (P2)}, solving problem \eqref{eqn:main-t-pf-opt} with $\mb r = \mb Y \mb e_j$ recovers {\em some} row of $\mb X$ whenever (i) $1 \le \| \mb X \mb e_j \|_0 \le s$ and (ii) $|b|_{(2)}/|b|_{(1)} \le 1 - \gamma$. Let
\begin{eqnarray}
\event_1(j) &=& \textbf{event} \{ \; \| \mb X \mb e_j \|_0 > 0 \; \}, \\
\event_2(j) &=& \textbf{event} \{\; \| \mb X \mb e_j \|_0 \le s\; \}, \\
\event_3(j) &=& \textbf{event} \{\; |b|_{(2)}/|b|_{(1)} \le 1 - \gamma\; \}.
\end{eqnarray}
Let $\event(i,j)$ be the event that these three properties are satisfied, and the largest entry of $\mb b = \mb X \mb e_j$ occurs in the $i$-th entry: 
$$\event(i,j) \;=\; \event_1(j) \,\cap \, \event_2(j) \, \cap \, \event_3(j) \, \cap \, \textbf{event}\left\{ \; |b|_{(1)} = |b_i| \; \right\}.$$
If the matrix $\mb X$ satisfies {\bf (P1)}-{\bf (P2)}, then on $\event(i,j)$, \eqref{eqn:main-t-pf-opt} recovers the $i$-th row of $\mb X$. Moreover, because $\event(i,j)$ only depends on the $j$-th column of $\mb X$, the events $\event(i,1) \dots \event(i,p)$ are mutually independent. By symmetry, for each $i$, $$
\P[ \event(i,j) ] \;=\; \frac{1}{n} \, \P[ \event_1(j) \cap \event_2(j) \cap \event_3(j) ].$$
The random variable $\| \mb X \mb e_j \|_0$ is distributed as a $\mathrm{Binomial}(n,\theta)$. So, 
$\P[ \event_1(j)^c ] = (1-\theta)^n \le (1- 2/n)^n \le e^{-2}$.
The binomial random variable $\| \mb X \mb e_j \|_0$ has expectation $\theta n$. Since $s \ge 6 \theta n$, by the Markov inequality, $\P[  \event_2(j)^c ] < 1/6$. Finally, by Lemma \ref{lem:2nd-our-form}, $\P \left[ \event_3(j)^c \mid \event_1 \right]  \le 1/2$. 
Moreover, 
\begin{eqnarray*}
\P[ \event_1(j) \cap \event_2(j) \cap \event_3(j) ] &\ge& 1 - \P[\event_1(j)^c] - \P[\event_2(j)^c] - \P[\event_3(j)^c \mid \event_1 ] \\ &\ge& 1 - e^{-2} - 1/6 - 1/2 \quad\doteq\quad \zeta.
\end{eqnarray*}
The constant $\zeta$ is larger than zero (one can estimate $\zeta \approx .198$).  For each $(i,j)$, $$\P\left[ \, \event(i,j)\, \right] \;\ge\; \frac{\zeta}{n}.$$
Hence, the probability that we fail to recover all $n$ rows of $\mb X$ is bounded by 
\begin{eqnarray*}
\P[ \mb X \; \text{does not satisfy {\bf (P1)}-{\bf (P2)}} ] + \sum_{i = 1}^n \P[ \cap_j \event(i,j)^c ] &\le& 7 p^{-10} + n ( 1 - \zeta / n )^p  \\
&\le& 7 p^{-10} + \exp\left( -\frac{\zeta p }{n} + \log n \right).
\end{eqnarray*}
Provided $\frac{\zeta p}{n} \ge \log n + 10 \log p$, the exponential term is bounded by $p^{-10}$. When $c_1$ is chosen to be a sufficiently large numerical constant, this is satisfied. 
\end{proof}

\subsection{The Two-Column Case}

The proof of Theorem \ref{thm:correct-tc} follows along very similar lines to Theorem \ref{thm:correct}. The main difference is in the analysis of the gaps. 
\vspace{.1in}

\begin{proof} We will apply Lemmas \ref{lem:reduction-new} and \ref{lem:extreme-sparse-new}. For Lemma \ref{lem:extreme-sparse-new}, we will set $\gamma = 1/2$, and $s = 12\theta n + 1$. Then $\theta s = 12 \theta^2 n$. Under our assumptions, provided the numerical constant $\alpha$ is small enough, $\theta s \le \gamma / 8$. Moreover, the hypotheses of Lemma \ref{lem:reduction-new} are satisfied, and so $\mb X$ has property {\bf (P1)} with probability at least $1-3 p^{-10}$. Provided $\alpha$ is sufficiently small, $s < 1/8 \theta$, as demanded by Lemma \ref{lem:reduction-new}. 

For each $j \in [p]$, let $\Omega_j = \mathrm{supp}(\mb X \mb e_j)$. Let $$\event_{\Omega,j,k}(i) \;\doteq\; \textbf{event}\left\{ |\Omega_j \cup \Omega_k| \le s, \;\text{and} \; \Omega_j \cap \Omega_k = \{i\} \right\}.$$
Hence, on $\event_{\Omega,j,k}(i)$, the vectors $\mb X \mb e_j$ and $\mb X \mb e_k$ overlap only in entry $i$. The next event will be the event that the $i$-th entries of $\mb X \mb e_j$ and $\mb X \mb e_k$ are the two largest entries in the combined vector $$\mb h_{jk} \doteq \left[ \begin{array}{c} \mb X \mb e_j \\ \mb X \mb e_k \end{array} \right],$$
and their signs agree. More formally:
$$\event_{X,j,k}(i) \;=\; \textbf{event}\left\{ \{ |\mb h_{jk}|_{(1)}, |\mb h_{jk}|_{(2)} \} = \{ X_{ij}, X_{ik}\} \quad \text{and} \quad \mathrm{sign}(X_{ij}) = \mathrm{sign}(X_{ik}) \, \right\}.$$
If we set $\mb r = \mb Y \mb e_j + \mb Y \mb e_k$, and hence $\mb b = \mb A^{-1} \mb r = \mb X \mb e_j + \mb X \mb e_k$, then on $\event_{X,j,k}(i) \cap \event_{\Omega,j,k}(i)$, the largest entry of $\mb b$ occurs at index $i$, and $|\mb b|_{(2)} / | \mb b|_{(1)} \le 1/2$. 

Hence, if $\mb X$ satisfies {\bf (P1)}-{\bf (P2)}, on $\event_{\Omega,j,k}(i) \cap \event_{X,j,k}(i)$, the optimization
\begin{equation}
\text{minimize} \quad \| \mb w^T \mb Y \|_1 \quad \text{subject to} \quad \mb r^T \mb w = 1
\end{equation}
with $\mb r = \mb Y \mb e_j + \mb Y \mb e_k$ recovers the $i$-th row of $\mb X$. The overall probability that we fail to recover some row of $\mb X$ is bounded by 
\begin{eqnarray*}
\P[ \mb X \; \text{does not satisfy {\bf (P1)}-{\bf (P2)}} ] + \sum_{i = 1}^n \P\left[ \cap_{l=1}^{p/2} (\event_{\Omega,j(l),k(l)}(i) \cap \event_{X,j(l),k(l)}(i))^c \right]
\end{eqnarray*}
Let $\Omega_j' = \Omega_j \setminus \{ i \}$, and $\Omega_k' = \Omega_k \setminus \{i\}$. Then we have 
\begin{eqnarray*}
\P\left[ \, \event_{\Omega,j,k}(i) \, \right] &=& \P\left[ \, i \in \Omega_j \;\text{and}\; i \in \Omega_k \, \right]\; \P\left[ \, |\Omega_{j}' \cup \Omega_k'| \le s-1 \;\, \text{and} \;\, \Omega_j' \cap \Omega_k' = \emptyset \, \right] \\
 &\ge& \theta^2 \left( 1 - \P\left[ \, |\Omega_j' | > \frac{s-1}{2} \, \right] - \P\left[ \, |\Omega_k' | > \frac{s-1}{2} \, \right] - \P\left[\, \Omega_j' \cap \Omega_k' \ne \emptyset \, \right] \right) \\
&=& \theta^2 \left( \P\left[\, \Omega_j' \cap \Omega_k' = \emptyset \, \right] - \P\left[ \, |\Omega_j' | > \frac{s-1}{2} \, \right] - \P\left[ \, |\Omega_k' | > \frac{s-1}{2} \, \right] \right) \\
&\ge& \theta^2 \left( (1-\theta^2)^{n-1} - 1/3 \right),
\end{eqnarray*}
where in the final line we have used the Markov inequality: $$\P\left[ |\Omega_j'| > \frac{s-1}{2} \right] \;\le\; \frac{2\E[ |\Omega_j'| ]}{s-1} \;=\; \frac{2 \theta (n-1)}{s-1} \;\le\; \frac{1}{6}.$$ 
Since $\theta < \alpha / \sqrt{n}$, we have $(1-\theta^2)^{n-1} \ge 1 - \alpha^2$. Provided $\alpha$ is sufficiently small, this quantity is at least $2/3$. Hence, we obtain
\begin{equation}
\prob{ \event_{\Omega,j,k}(i) } \;\ge\; \frac{\theta^2}{3}. 
\end{equation}
For $\event_{X,j,k}(i)$, we calculate 
\begin{equation}
\P\left[ \event_{X,j,k}(i) \mid \event_{\Omega,j,k}(i) \right] \;\ge\; \frac{1}{2 \binom{s}{2}} \;\ge\; \frac{1}{2 (12 \theta n + 1)^2 },
\end{equation}
and
\begin{equation}
\prob{ \event_{\Omega,j,k}(i) \cap \event_{X,j,k}(i) } \;\ge\; \frac{\xi}{n^2},
\end{equation}
for some numerical constant $\xi > 0$. Hence, the overall probability of failure is at most 
\begin{equation}
7 p^{-10} + n \left( 1- \frac{\xi}{n^2} \right)^{p/2} \;\le\; 7 p^{-10} + \exp\left( - \frac{\xi p}{2 n^2}  + \log n \right).
\end{equation}
Provided $\frac{\xi p}{2 n^2} > \log n + 10 \log p$, this quantity can be bounded by $p^{-10}$. Under our hypotheses, this is satisfied. 
\end{proof}

\section{Upper Bounds: Proof of Theorem \ref{thm:ub}}

\begin{proof}
For technical reasons, it will be convenient to assume that the entries of the random matrix $\mb X$ are bounded by $\tau$ almost surely. We will set $\mb X' = \trunc{\tau} \mb X$, and set $\tau = \sqrt{24 \log p}$, so that with probability at least $1- 2 p^{-10}$, $\mb X' = \mb X$. We then prove the result for $\mb X'$. Notice that we may write $\mb X' = \mb \Omega \odot \mb R'$, with $\mb R'$ iid subgaussian, and $\mu' \doteq \E[ | R'_{ij} | ] \ge 1/20$. 

Applying a change of variables $\mb z = \mb A^T \mb w$, $\mb b = \mb A^{-1} \mb r$, we may analyze the equivalent optimization problem
\begin{equation} \label{eqn:ub-opt}
\text{minimize} \; \| \mb z^T \mb X' \|_1 \quad \text{subject to} \quad \mb b^T \mb z = 1.
\end{equation}
Let $\mb v = \frac{\mathrm{sign}(\mb b)}{\| \mb b \|_1}$. This vector is feasible for \eqref{eqn:ub-opt}. We will compare the objective function obtained at $\mb z = \mb v$ to that obtained at any feasible one-sparse vector $\mb z = \mb e_i / b_i$. Note that 
$$\| \mathrm{sign}(\mb b)^T \mb X'  \|_1 \;=\; \sum_{k = 1}^p \left| \sum_{i=1}^n X'_{ik} \, \mathrm{sign}(b_i) \right| \;\doteq\; \sum_k Q_k \;=\; Q_j+\sum_{k\neq j} Q_k.$$ 
The random variable $Q_j$ is just the $\ell^1$ norm of $\mb X'_j$, which is bounded by $n \tau$. The random variables $(Q_k)_{k \ne j}$ are conditionally independent given $\mb b$. We have 
$$
\E\left[\, Q_k^2 \, \mid \mb b \, \right] \;=\; \theta \| \mb b \|_0,
$$
By Bernstein's inequality, 
\begin{equation}
\P\left[ \sum_{k \ne j} Q_k \ge \E\left[ \sum_{k\ne j} Q_k \mid \mb X_j \right] + t \mid \mb X_j' \right] \;\le\; \exp\left( \frac{-t^2}{2 p \theta \| \mb X_j' \|_0 + 2 \tau \| \mb X_j' \|_0 / 3 }\right).
\end{equation}
If the constant $c$ in the statement of the theorem is sufficiently large, then $p \theta > c_1 \tau$. We also have $\E\left[ Q_k \mid \mb X_j' \right] \le \sqrt{\theta \| \mb X_j' \|_0}$. Simplifying and setting $t = p\sqrt{ \theta \| \mb X_j' \|_0}$, we obtain
\begin{equation}
\P\left[ \sum_{k \ne j} Q_k \ge 2 p\sqrt{\theta \| \mb X_j' \|_0} \mid \mb X_j' \right] \;\le\; \exp\left( - c_2 p\right).
\end{equation}
The variable $Q_j$ is bounded by $n\tau$. Moreover, a Chernoff bound shows that 
\begin{equation}
\P\left[ \| \mb X_j' \|_0 \;\ge\; 4 \theta n \right] \;\le\; \exp\left( -  3\theta n \right) \;\le\; \exp\left( - 3 \beta \sqrt{n\log n} \right).
\end{equation}
So, with overall probability at least $1-\exp\left( - c_2 p \right) + \exp\left( - 3 \beta \sqrt{n \log n} \right)$, 
\begin{equation}
\| {\mb X'}^T \mathrm{sign}(\mb b) \|_1 \;\le\;  4 \theta p \sqrt{n} + n \tau \;\le\; 5 \theta p \sqrt{n},
\end{equation}
where in the final inequality, we have used our lower bound on $p$. Using that $\mb b = \mb X' \mb e_j$ is Bernoulli-Subgaussian, we next apply Lemma \ref{lem:M-norm-bound-bg} with $\delta = 1/2$ to show that 
\begin{align}
\P\left[ \, \|\mb b\|_1 \;\le\; \tfrac{1}{2} \mu' \theta n \, \right] \;\leq\; 4 \exp(-c_3 \theta n)\;\leq\; 4\exp(-c_3 \sqrt{\beta n\log n} )
\end{align}
On the other hand, by Lemma \ref{lem:M-norm-bound-bg},
 \begin{align}
\P\left[ \min_{i} \| \mb e_i^T \mb X'  \|_1 \;\le\; \frac{\theta p \mu'}{2} \right] \le  4 \exp( - c_3 \theta p + \log n). 
 \end{align}
and using the subgaussian tail bound,
\begin{equation}
\P\left[ \| \mb b \|_\infty > t \right] < 2 n \exp(-t^2 / 2 ). 
\end{equation}
Set $t = \sqrt{10 \log n}$ to get that $\| \mb b \|_\infty \le \sqrt{10 \beta \log n}$ with probability at least $2 n^{1- 5 \beta}$.

Hence, with overall probability at least 
$$
1 -  \exp\left(- c_2 p \right) - \exp\left( - 3 \beta \sqrt{n \log n} \right) - 4 \exp\left( - c_3 \theta p + \log n \right)  - 2 n^{1-5\beta}
$$
we have 
\begin{equation} \label{eqn:avg-bound}
\left\| {\mb X'}^T \frac{\mathrm{sign}(\mb b)}{\|\mb  b \|_1} \right\|_1 \;\le\; \frac{5 \theta p \sqrt{n}}{\tfrac{1}{2} \mu \theta n} \;\le\; \frac{c_5 p}{\sqrt{n}}.
\end{equation}
while
\begin{equation}
\left\| {\mb X'}^T \frac{ \mb e_i}{b_i} \right\|_1 \;\ge\; \frac{\| {\mb X'}^T \mb e_i \|_1}{\| \mb b \|_\infty} \;\ge\; \frac{\theta \mu p}{2 \sqrt{10} \sqrt{\log n}} \;\ge\; \frac{ c_4 \theta p }{\sqrt{\beta \log n}} \;\ge\; \frac{c_4 \sqrt{\beta} p }{\sqrt{n}} . \label{eqn:point-bound}
\end{equation}
If $\beta$ is sufficiently large, \eqref{eqn:point-bound} is larger than \eqref{eqn:avg-bound}, and the algorithm does not recover any row of $\mb X'$. Since with probability at least $1-2p^{-10}$, $\mb X' = \mb X$, the same holds for $\mb X$, with an additional failure probability of $2p^{-10}$. 
\end{proof}

\section{Gaps in Gaussian Random Vectors}

In this section, we consider a $d$-dimensional random vector $\mb r$, with entries iid $\mc N(0,1)$. We let 
$$s(1) \ge s(2) \ge \dots \ge s(d)$$
denote the order statistics of $|\mb r|$. We will make heavy use of the following facts about Gaussian random variables,
  which may be found in~\cite[Section VII.1]{Feller}.
\begin{lemma}\label{lem:feller}
Let $x$ be a Gaussian random variable with mean 0 and variance $1$.
Then, for every $t > 0$
\[
  \frac{1}{t} p (t) \geq \P\left[x \geq t\right] \geq 
  \left(\frac{1}{t} - \frac{1}{t^{3}} \right) p (t),
\]
where
\[
  p (t) = \frac{1}{\sqrt{2 \pi}} \exp \left( -t^{2}/2 \right).
\]
\end{lemma}

\begin{lemma}\label{lem:max-ub} For any $d \ge 2$, 
\[
 \P\left[ s(1) > 4 \sqrt{\log d} \right] \;\le\; d^{-3}.
\]
\end{lemma}
\begin{proof}
Set $t = 4 \sqrt{\log d}$ in Lemma \ref{lem:feller} to show that 
\begin{equation}
\P\left[ | r(i) | \ge t \right] \le \frac{1}{4 \sqrt{2 \pi \log d}} d^{-4}. 
\end{equation}
The denominator is larger than one for any $d \ge 2$; bounding it by $1$, and taking a union bound over the $d$ elements of $\mb r$ gives the result. 
\end{proof}

\begin{lemma}\label{lem:max-lb}
For $d$ larger than some constant
\[
  \P\left[ s (1) < \frac{\sqrt{\log d}}{2} \right]
 \leq 
  \exp \left( -\frac{d^{7/8}}{4 \sqrt{\log d}} \right). 
\]
\end{lemma}
\begin{proof}
Set
\[
t = \frac{1}{2} \sqrt{\log d}
\]
For $d$ sufficiently large, we may use
  Lemma~\ref{lem:feller} to show that the probability that
  a Gaussian random variable of variance 1 and any mean has absolute
  value greater than $t$ is at least 
  $$\frac{1}{4\sqrt{ \log d}} d^{-1/8}.$$
Thus, the probability that every entry of $\mb r$ has absolute
  value less than $t$ is at most  
\[
  \left(1 - d^{-1/8} / 4 \sqrt{\log d} \right)^{d}
\leq \exp \left( -\frac{d^{7/8}}{4 \sqrt{\log d}} \right).
\]
\end{proof}

We now examine the gap between the largest and second-largest entry of $\mb r$.
We require the following fundamental fact about Gaussian random variables.
\begin{lemma}\label{lem:univariateGap}
Let $x$ be a Gaussian random variable with variance $1$ and arbitrary mean.
Then for every $t > 0$ and $\alpha   > 0$,
\[
  \P\left[\abs{x}\leq t+\alpha  \big| \abs{x} \geq t\right] \leq 3 \alpha \max (t,3).
\]
\end{lemma}
\begin{proof}
Assume without loss of generality that the mean of $x$ is positive.
Then,
\begin{align*}
\P\left[\abs{x}\leq t+\alpha  \big| \abs{x} \geq t\right]
& = 
\frac{
\P\left[t \leq \abs{x}\leq t+\alpha \right]
}{
\P\left[\abs{x}\geq t \right]
}\\
& \leq 
\frac{
2 \P\left[ t \leq x \leq t+\alpha \right]
}{
\P\left[ \abs{x}\geq t \right]
}\\
& \leq 
\frac{
2 \P\left[ t \leq x \leq t+\alpha \right]
}{
\P\left[ x \geq t \right]
}.
\end{align*}
One can show that this ratio is maximized when the mean
  of $x$ is in fact $0$, given that it is non-negative.
Similarly, the ratio is monotone increasing with $t$.
So, if $t \leq 3$, we will upper bound this probability by
  the bound we obtain when $t = 3$.
When the mean of $x$ is $0$ and $t \geq 3$,
 Lemma~\ref{lem:feller} tells us that
\[
\P\left[x \geq t\right]
 \geq \frac{1}{\sqrt{2 \pi}t} \left(1 - \frac{1}{t^{2}}  \right) \exp (-t^{2}/2)
 \geq \frac{8}{9 \sqrt{2 \pi}t} \exp (-t^{2}/2).
\]
We then also have
\[
\P\left[ t \leq x \leq t+\alpha \right]
=
\int_{t}^{t+\alpha} \frac{1}{\sqrt{2 \pi}} \exp (-x^{2} / 2) dx
\leq 
\frac{\alpha}{\sqrt{2 \pi}} \exp (-t^{2} / 2) .
\]
The lemma now follows from $2 \cdot 9 / 8 \leq 3$.
\end{proof}

\begin{lemma}\label{lem:2ndgap} Let $d \ge 2$. Then for every $\alpha  > 0$, 
\[
  \P\left[ s (1) - s (2) < \alpha  \right]
\leq 12 \alpha \sqrt{\ln d} + d^{-2}.
\]
\end{lemma}
\begin{proof}
Let $M_{i}$ be the event that $q (i)$ is the largest entry of $q$ in absolute value.
As the events $M_{i}$ are disjoint, we know that the sum of their probabilities is $1$. Let $R_{i}$ be the event that $\max_{j \not = i} \abs{q (j)} \leq 4 \sqrt{\ln d}$. From Lemma~\ref{lem:max-ub} we know that the probability of $\mnot (R_{i})$ is at most $1/d^{3}$.
Let $G_{i}$ be the event that $M_{i}$ holds but that the gap between $q (i)$
  and the second-largest entry in absolute value is at most $\alpha$. From Lemma~\ref{lem:univariateGap}, we know that
\[
  \P\left[ G_{i} | R_{i} \mand M_{i} \right] \leq 12 \alpha \sqrt{\ln d}.
\]
The lemma now follows from the following computation.
\begin{align*}
  \P\left[ s (1) - s (2) < \alpha  \right]
& = 
  \sum_{i} \P\left[ G_{i} \right]
\\
& = 
  \sum_{i} \P\left[ G_{i} \mand M_{i} \right]
\\
& = 
  \sum_{i} \P\left[ G_{i} \mand M_{i} \mand R_{i} \right]
+
  \sum_{i} \P\left[ G_{i} \mand M_{i} \mand \mnot (R_{i}) \right].
\end{align*}
We have
\[
  \sum_{i} \P\left[ G_{i} \mand M_{i} \mand \mnot (R_{i}) \right]
\leq 
  \sum_{i} \P\left[ \mnot (R_{i}) \right]
\leq 
1/d^{2},
\]
and
\begin{align*}
  \sum_{i} \P\left[ G_{i} \mand M_{i} \mand R_{i} \right]
& =
  \sum_{i} \P\left[ M_{i} \mand R_{i} \right] \P\left[ G_{i} | M_{i} \mand R_{i} \right]
\\
& \leq 
  \sum_{i} \P\left[ M_{i} \right] \P\left[ G_{i} | M_{i} \mand R_{i} \right]
\\
& \leq 
  \sum_{i} \P\left[ M_{i} \right] 12 \alpha \sqrt{\ln d}
\\&=
12 \alpha \sqrt{\ln d}.
\end{align*}
\end{proof}

\begin{lemma} \label{lem:2nd-our-form} There exists $\alpha_0 > 0$ such that for any $\alpha < \alpha_0$ and $2 \le d \le n$ the following holds.  If $\mb r$ is a $d$-dimensional random vector with independent standard Gaussian entries
$$s(1) \ge s(2) \ge \dots \ge s(d)$$
are the order statistics of $|\mb r|$, then 
\begin{equation}
\P\left[ 1 - \frac{s(2)}{s(1)} \;<\; \frac{\alpha}{\log n} \right] \;<\; \frac{1}{2}.
\end{equation}
\end{lemma}
\begin{proof}
For any $\alpha > 0$, $t > 0$, we have 
\begin{eqnarray*}
\P\left[ \, 1 - \frac{s(2)}{s(1)} \;\ge\; \frac{\alpha}{\log n} \, \right] &\ge& \P\left[ \, s(1) - s(2) \;\ge\; \frac{\alpha t}{\sqrt{\log n}} \;\;\text{and}\;\; s(1) \le t \sqrt{\log n} \, \right] \\
&\ge& 1 - \P\left[ \, s(1) - s(2) \;<\; \frac{\alpha t}{\sqrt{\log n}} \,\right] - \P\left[ \, s(1) \;>\; t \sqrt{\log n} \, \right] \\
&\ge& 1 - \P\left[ \, s(1) - s(2) \;<\; \frac{\alpha t}{\sqrt{\log n}} \,\right] - \P\left[ \, s(1) \;>\; t \sqrt{\log d} \, \right].
\end{eqnarray*}
Setting $t = 4$ and applying Lemmas \ref{lem:2ndgap} and \ref{lem:max-ub}, we have
\begin{eqnarray*}
\P\left[ \, 1 - \frac{s(2)}{s(1)} \;\ge\; \frac{\alpha}{\log n} \, \right] &\ge& 1 - 48 \alpha \sqrt{\frac{\log d}{\log n}} - d^{-2} - d^{-3}.
\end{eqnarray*}
Choosing $\alpha_0$ sufficiently small (and noting that for $d \ge 2$, $d^{-2} + d^{-3} \le 3/8$) completes the proof. 
\end{proof}

\end{document}